\documentclass[pdflatex,sn-mathphys-num]{sn-jnl}

\usepackage{graphicx}%
\usepackage{multirow}%
\usepackage{amsmath,amssymb,amsfonts}%
\usepackage{amsthm}%
\usepackage{mathrsfs}%
\usepackage[title]{appendix}%
\usepackage{xcolor}%
\usepackage{textcomp}%
\usepackage{manyfoot}%
\usepackage{booktabs}%
\usepackage{algorithm}%
\usepackage{algorithmicx}%
\usepackage{algpseudocode}%
\usepackage{listings}%
\usepackage{mathtools}

\usepackage{tikz-cd}
\newtheorem{lemma}{Lemma}

\newcommand{\soo}{\mathfrak{so}(3)}
\newcommand{\SE}{\mathbf{SE}(3)}
\newcommand{\se}{\mathfrak{se}(3)}
\newcommand{\SO}{\mathbf{SO}(3)}
\newcommand{\inv}{\textbf{inv}}
\newcommand{\Ad}{\textbf{Ad}}
\newcommand{\ad}{\textbf{ad}}
\newcommand{\spn}{\textbf{span}}
\newcommand{\R}{\mathbb{R}}
\newcommand{\g}{\mathbf{g}}
\newcommand{\A}{\mathcal{A}}
\newcommand{\0}{\mathbf{0}}

\theoremstyle{thmstyleone}%
\newtheorem{theorem}{Theorem}
\newtheorem{proposition}[theorem]{Proposition}%

\theoremstyle{thmstyletwo}%
\newtheorem{remark}{Remark}%

\theoremstyle{thmstylethree}%

\begin{document}

\title{Geometrically Exact Hard Magneto-Elastic Cosserat Shells: Static Formulation for Shape Morphing}


\author[1]{\fnm{Mohammadjavad} \sur{Javadi }}\email{javadjavadisigaroudi@cmail.carleton.ca}

\author*[2]{\fnm{Robin} \sur{Chhabra}}\email{robin.chhabra@torontomu.ca.}


\affil[1]{\orgdiv{Department of Mechanical and Aerospace Engineering}, \orgname{Carleton University}, \orgaddress{\street{Ottawa, K1S 5B6, Ontario, Canada}}} 

\affil*[2]{\orgdiv{Department of Mechanical, Industrial, and Mechatronics Engineering}, \orgname{Toronto Metropolitan University}, \orgaddress{\street{Toronto, M5B 2K3, Ontario, Canada}}}



\abstract{Cosserat rod theory is the popular approach to modeling ferromagnetic soft robots as 1-Dimensional (1D) slender structures in most applications, such as biomedical. However, recent soft robots designed for locomotion and manipulation often exhibit a large width-to-length ratio that categorizes them as 2D shells. For analysis and shape-morphing control purposes, we develop an efficient coordinate-free static model of hard-magnetic shells found in soft magnetic grippers and walking soft robots. The approach is based on a novel formulation of Cosserat shell theory on the Special Euclidean group ($\SE$). 
The shell is assumed to be a 2D manifold of material points with six degrees of freedom (position \& rotation) suitable for capturing the behavior of a uniformly distributed array of spheroidal hard magnetic particles embedded in the rheological elastomer. The shell's configuration manifold is the space of all smooth embeddings $\mathbb{R}^2\rightarrow\SE$. According to a novel definition of local deformation gradient based on the Lie group structure of $\SE$, we derive the strong and weak forms of equilibrium equations, following the principle of virtual work. We extract the linearized version of the weak form for numerical implementations.  The resulting finite element approach can avoid well-known challenges such as singularity and locking phenomenon in modeling shell structures.
The proposed model is analytically and experimentally validated through a series of test cases that demonstrate its superior efficacy, particularly when the shell undergoes severe rotations and displacements.}

\keywords{Cosserat shell, Lie groups, Soft continuum robots, Magnetoelasticity, Large deformation}



\maketitle

\section{Introduction}\label{sec1}
Soft robots are designed to benefit from one of the most salient features of living organisms, i.e., their deformable body. Through this feature, they can continuously change their shape or properties to adapt to the environment and navigate or manipulate safely in confined spaces \cite{rus2015,marchese2016}.  
Soft continuum robots are extensively used in a multitude of applications such as surgery \cite{burgner2015, runciman2019}, wearable rehabilitation devices \cite{park2014d}, drug delivery, gripper design \cite{cianchetti2018}, and space exploration \cite{zhang2023progress}. For effective real-time control of their behavior in partially known environments, soft robots' planning and perception modules rely on accurate and fast numerical models. 

A leading technique for modeling soft robots is based on the assumption of Constant Curvature (CC) kinematics. This method breaks down the continuously deformable body into a finite number of CC arc segments \cite{webster2010}. While the CC approach is beneficial for representing large deformations, it has drawbacks when dealing with certain geometries and designs \cite{renda2022geome}, and it can encounter local singularities and reduced accuracy under severe body and external loads \cite{sadati2017control}.
Another popular approach in the field of soft robotics originates from the works of the Cosserat brothers \cite{cosserat1909} and has been developed further by Simo under the title of the Geometrically-Exact (GE) model\cite{simo1988dynamics, renda2022geometrically}. This approach presents the geometrical description of 1D soft robots within the framework of continuum mechanics, based on the assumptions of Cosserat theory \cite{simo1988dynamics, simo1990}.
For instance, Sadati et al. \cite{sadati2017control} proposed a Cosserat rod model for multisegment continuum manipulators. Renda et al. \cite{renda2018unified} utilized the Cosserat rod model and a similar 1D model of an axisymmetric shell for modeling of underwater robots.


Besides the geometrical model of soft robots, the actuation mechanisms play important roles in their design, fabrication, and modeling. Soft robots are generally made of smart materials to facilitate actuation, such as shape memory alloy, ferromagnetic, or dielectric soft materials \cite{liu2021robotic,guseinov2020,kotikian2019,zhang2021liqu,alapan2020re,shah2021soft}. Among various categories of actuators, ferromagnetic mechanisms offer safe, simple, and efficient manipulation methods. Ferromagnetic soft robots can be developed by embedding small-scale magnetic particles inside hyperelastic polymeric shells. These types of materials have many applications in different areas. For example, Dad Ansari et al. developed a programmable magnetization method for printing soft small-scale robots with a non-uniform magnetization profile \cite{ansari20233d}. Alapan et al. proposed a different structure of magnetic soft materials with a heating approach for reorientation of the magnetic field inside an elastomer \cite{alapan2020re}. Kim et al. presented a submillimeter-scale soft magnetic robot with embedded magnetic particles inside a 1D rod for active steering and navigating capabilities \cite{kim12019}.



The application of the Lie group $\SE$ in the modeling of rigid robots, due to its singularity-free geometric representation of motion  \cite{murray1994}, has attracted roboticists to extend this approach to soft robots with infinite degrees of freedom. Additionally, describing the dynamics on $\SE$ facilitates the study of robotic systems that can include both soft and rigid bodies \cite{samei2024surfr}. 
In this context, several studies have extended the modeling of 1D rod structures using the structure of $\SE$ to capture the behaviour of soft grippers and robots with some degrees of locomotion\cite{hussain2018modeling, grazioso2019, boyer2020dynamics, xun2024cosse}. One significant feature of this geometric approach is its ability to investigate translational and rotational motion as coupled quantities, a capability not captured by classical continuum mechanics. This natural coupling in the work of Sonneville et al. \cite{sonneville2014} results in a compact and shear-locking-free finite element implementation. In recent publications, Renda and Talegon \cite{renda2024dynamics} introduced a novel strain parameterization of 1D Cosserat rod structures on $\SE$ to develop a fast model of robots suitable for their control. It is important to note that while strain parameterization in 1D problems reduces the degrees of freedom, it complicates 2D shell models due to the need for an appropriate interpolation function to satisfy the Cosserat shell compatibility equation.  
\begin{figure}[h]
\includegraphics[scale=0.5]{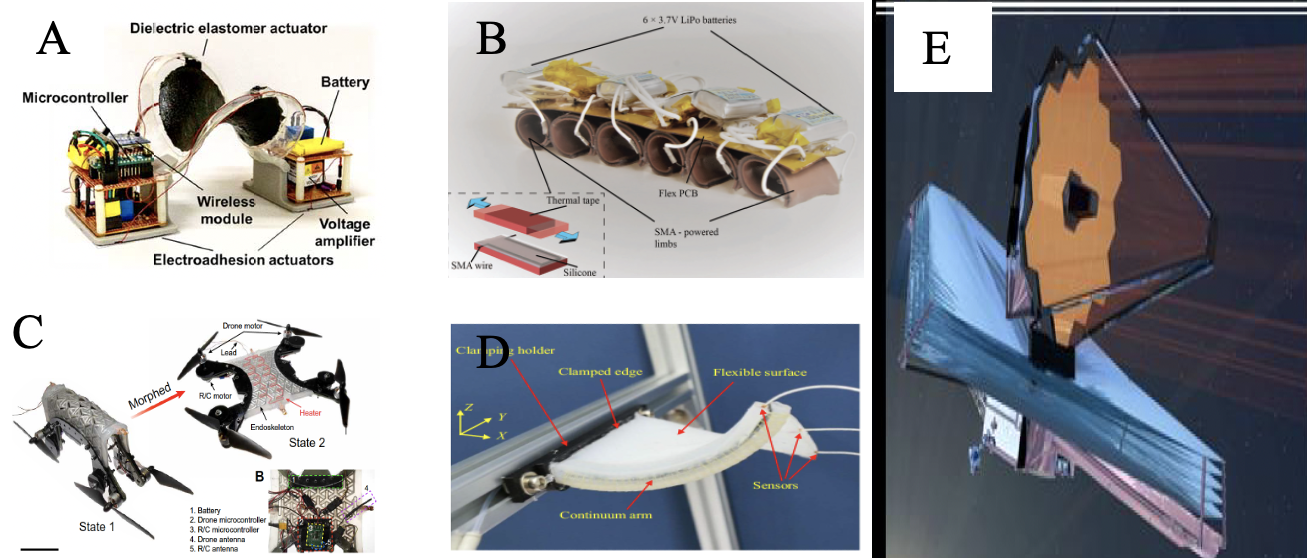}
\centering
\caption{\textbf{Shell structures for shape morphing soft robots:}(A) Spherical dielectric elastomer shell in the untethered soft robot \cite{cao2018untet}. (B) Cylindrical shell structure in the modeling of a caterpillar-inspired robot \cite{goldberg2019planar}. (C) Deformable shell structure in a multifunctional morphing drone \cite{hwang2022shape}. (D) Soft flexible surface for robotic applications \cite{habibi2020lumped}. (E) The JWST sunshield in its deployed position. \cite{arenberg2008design}.}
\label{fig2s}
\end{figure}
Thin shell structures find extensive applications in various industries. For instance, they are used in aerospace for aircraft and missile fuselages \cite{chapkin2020design}, in nanotechnology for modeling nanotubes \cite{peng2008measurements}, in biology for representing soft growing surface tissues \cite{goriely2005differential}, and in soft robotics to capture a model of 2D structures \cite{shah2021soft}.
Fabrication of flat or 2D active shells is often easier than other forms of smart structures and they can transform into complex shapes by changing the surface curvature at arbitrary points, an interesting challenge pursued by scientists \cite{baek2018,klein2007,lang1996,kim2012d,klein2007}. For example, to fit the James Webb Space Telescope (JWST) into a rocket, a foldable sunshield membrane was used \cite{arenberg2008design} (See figure \ref{fig2s}(E)).  
The design of available soft robots for locomotion and manipulation often features a large width-to-length ratio, categorizing them as 2D shells (See figure \ref{fig2s}). Representing the body of soft robots as 2D shell structures also enables the capture of a wide range of potential deformations and motions that are not achievable otherwise. For instance, Nathaniel et al. \cite{goldberg2019planar} studied the behavior of a caterpillar-inspired soft robot using Bergou et al.'s rod model \cite{bergou2008discrete}, which can be considered a shell model for capturing asymmetric deformation due to uneven load distribution along the robot's width (See \ref{fig2s}(B)). The geometric design and loading conditions presented in \cite{xu2019millimeter} indicate that for modeling millimeter-scale flexible robots, a 2D model is required instead of a 1D Cosserat rod model. Similar issues are also observable in other research works \cite{mao2022ultrafast, ansari20233d}.
The lack of an appropriate shell model that can be utilized for various soft robot designs and loading conditions, and implemented in a fast and accurate numerical framework was the driving force behind this research. In the next paragraph, we will review different shell theories within the framework of continuum mechanics.


In the framework of continuum mechanics, a shell is a special three-dimensional body whose dimension of the body along the normals (thickness) is small. One of the complete theories for studying the bending behavior of thin elastic plates was derived by Kirchhoff \cite{kirchhoff1850}, in which straight lines perpendicular to the mid-surface remain straight after deformation. Thirty-eight years later, Love\cite{love1892} presented a complete linear bending theory for thin elastic shells, excluding some assumptions from Kirchhoff's theory. There are many papers on the modeling of linear shell theory, but there is an exact nonlinear theory of shells under the Kirchhoff hypothesis by Budiansky which can consider geometric nonlinearities in the shell formulation \cite{budiansky1968}.
The manifestation of the non-classical continuum mechanics by the Cosserat brothers' idea opens a new way in the modeling of shell structures \cite{cosserat1909}. In this theory, the basic physical idea of classical continuum mechanics is endowed with extra kinematics degrees of freedom afforded by a triad of rigid vectors or directors attached at each material point. Another version of the Cosserat continuum theory was formulated by Eringen who proposed the idea of the material microstructure in solid mechanics \cite{eringen1976, eringen1999}. In Eringen's terminology, the Cosserat continuum mechanics is called \emph{micropolar} continuum mechanics. In both micropolar and Cosserat continuum theory, each material point has six degrees of freedom which are similar to a rigid body motion.  
Generally, there are two approaches to extracting shell formulations \cite{naghdi1973}. In the first approach, the two-dimensional balance equations of a shell are derived from the three-dimensional field equations by taking the integral over the shell thickness. Based on the shell's kinematics model, the resultant stress tensors are extracted as two-dimensional fields.  In the second approach, the balance equations are derived directly. The first approach is applicable to the classical Kirchhoff-Love shell model, while the second approach is suitable for Cosserat shells with multiple directors. Following the literature, all 6 Degree-of-Freedom micropolar shells are developed by a direct approach, and most of the Cosserat shell models only have a free director, not a complete triad (rotation). Furthermore, there is only one paper that tries to develop a Cosserat shell with free rotation, but it ignores the Lie group structure of SE(3) and only considers position and rotation separately.
The parameterization of finite rotations of the directors is an important challenge in Cosserat shell modeling. The way of describing the rotation matrix has a close relation to the appeared singularities in shell mechanics. A singularity-free parametrization of the rotation field was developed for a shell model with one inextensible director by Simo and his co-authors \cite{Simo1989, simo1990}, which ignores one independent degree of freedom, known as the \emph{drilling rotation}. Ibrahimbegovic and Frey \cite{ibrahimbegovic1994} presented a singularity-free nonlinear shell theory with drilling rotation based on the Rodrigues formula. The direct Cosserat shell theory with triad orthonormal directors,  provides a geometrical basis to capture all possible rotation actions, especially drilling motion at each material point \cite{sansour1995co, altenbach2010, altenbach2011mi}. The significant challenge for this shell model is the singularity in severe large deformation due to inappropriate parametrization of the rotation tensor.

All existing geometrical shell models have analyzed translation and rotation variables separately, using distinct interpolation functions for numerical integration. Additionally, in cases of severe large rotation deformations, these models have utilized quaternion parameters for rotation parametrization.
In this paper, we derive and numerically solve the equations of motion for an elastic Cosserat shell structure under mechanical and magnetic loading. The presented modeling can capture large deformation and rotation of the shell structure. This paper’s contributions are fourfold.

\begin{itemize}
 \item Taking advantage of the structure of the Lie group $\SE$ we present, for the first time, the mechanical formulation of a six-degree-of-freedom Cosserat shell capable of capturing drilling motion, unlike \cite{Simo1989,ibrahimbegovic1994,sansour1995co,arciniega2007tensor}.


 \item We develop the mechanical balance equations of the Cosserat shells in the presence of a magnetic body couple from the principle of virtual work. This approach uses a Lagrangian density function of the kinematic components in $\SE$, unlike in \cite{pezzulla2022,dadgar2023micropolar}.   
 \item We extract a singularity-free Lagrangian finite element technique to capture large deformations and rotations based on the weak form of the equilibrium equation. The terms are linearized within the framework of $\SE$ to extract a compact form of the material and geometric stiffness matrices, unlike in \cite{Simo1989,simo1990}. 

 \item We extend the idea of CC to the shell numerical modeling to alleviate shear-locking phenomenon without introducing additional numerical complexity, as done in \cite{simo1990, sansour1995co, ibrahimbegovic1994,dadgar2023micropolar,nebel2023geometrically,sauer2024simple}.

\end{itemize}

The structure of the article is as follows. In the next section, we present a brief review of the Lie group structure of $\SE$. Section \ref{sec:shell} develops the exact geometrical description of the proposed six-degree-of-freedom shell, the governing balance equations in both weak and strong forms, and the magnetic shell constitutive equation. Section \ref{sec:lin} introduces the admissible variations and the linearized form of the balance equations for numerical implementation. In Section \ref{sec:fem}, a novel finite element formulation is introduced. Section \ref{sec:exp} investigates various numerical examples and experimental results using the proposed model to demonstrate its accuracy. We include some concluding remarks in Section \ref{sec:conc}.


\section{Preliminaries}\label{sec:LieGroup}
\subsection{Notation}
Let $G$ be a matrix Lie group and $\mathfrak{g}$ be its Lie algebra with the Lie bracket defined via the matrix commutator operator $[\cdot,\cdot ]:\mathfrak{g}\times \mathfrak{g}\rightarrow \mathfrak{g}$. The relation between elements of $\mathfrak{g}$ and $G$ is through the group exponential map $\exp: \mathfrak{g} \rightarrow G$. The tangent and cotangent bundles of the Lie group $G$ are denoted by $TG$ and $T^{*}G$, respectively. The notation $(\cdot)^{*}$ is also used to identify the dual of a vector space or a vector (field), in this paper. A function $\mathbf{g}:\mathbb{R}^2 \rightarrow G$ belongs to the class $\mathcal{C}^{\infty}(\mathbb{R}^2 \rightarrow G)$, i.e., it is smooth, if it has an infinite number of derivatives. Here, the natural pairing between the elements of $T^{*}G$ and $TG$ are presented by $\Big<\cdot| \cdot \Big>$ and for elements of $\mathfrak{g}$ and $\mathfrak{g}^{*}$ we use $\Big<\cdot , \cdot \Big>$. The notations $(\cdot)^{-1}$ and $\text{det}(\cdot)$ represent the inverse and determinant of a square matrix, respectively. We denote the inverse of functions by $\inv(\cdot)$. Let $\mathcal{B}$ be a manifold, then a tensor of type $(p,q)$ at $X\in \mathcal{B}$ is a mapping $\mathbf{H}:\underbrace{T^{*}_{X}\mathcal{B}\times ... \times T^{*}_{X}\mathcal{B}}^{p\, \text{copies}}\times \underbrace{T_{X}\mathcal{B}\times ... \times T_{X}\mathcal{B}}^{q\, \text{copies}}\rightarrow \mathbb{R}$. The tensor product of two tensors $\mathbf{H}$ of type $(p,q)$ and $\mathbf{L}$ of type $(r,s)$ is denoted by $\mathbf{H}\otimes \mathbf{L}$ and it is of type $(p+r,q+s)$ \cite{marsden}.
\subsection{The Lie Group Structure of $\SE$}
The space of all rigid transformations between two coordinate frames is the Special Euclidean matrix Lie group
\begin{align}\label{E4}
\SE\coloneqq\Big\{\mathbf{g}=
\begin{pmatrix} 
\mathbf{R} & \mathbf{P}  \\
\mathbf{0}_{1\times 3} & 1  
\end{pmatrix}|\mathbf{R}\in \SO, \mathbf{P}\in \mathbb{R}^{3}\Big\},
\end{align}
where $\mathbf{P}$ is the translation component and the Lie group
\begin{align}\label{E1}
\SO\coloneqq\{\mathbf{R}:\mathbb{R}^{3}\rightarrow \mathbb{R}^{3}|\mathbf{R}^{T}=\mathbf{R}^{-1} \text{and}\quad \text{det}\mathbf{R}=1\}
\end{align}
contains all 3D rotations in the form of $3\times 3$ matrices $\mathbf{R}$. The tangent space at the identity element $\mathbf{I}_{4\times 4}\in \SE$ is isomorphic to the Lie algebra $\se$ as a vector space
\begin{align}\label{E5-11}
T_{\mathbf{I}_{4\times 4}}\SE\cong\se\coloneqq\Big\{\hat{\boldsymbol{\xi}}=
\begin{pmatrix} 
\hat{\boldsymbol{\omega}} & \mathbf{v}  \\
\mathbf{0}_{1\times 3} & 0  
\end{pmatrix}|\hat{\boldsymbol{\omega}}\in \soo, \mathbf{v}\in \mathbb{R}^{3}\Big\}.
\end{align}
We refer to elements of $\se$ as \emph{twists} or infinitesimal generators of $\SE$. 
With some abuse of notation, here we identify the elements of $\se$, using the vector space isomorphisms $\mathbb{R}^{3}\rightarrow \soo$ and $\mathbb{R}^{6}\rightarrow \se$, both denoted by the hat operator.
For all $v \in \mathbb{R}^{3}$, $\hat{v}$ is the unique $3\times 3$ skew-symmetric matrix such that $\hat{v}y=v\times y$ for every $y\in \mathbb{R}^{3}$.
The inverse of the hat operator is indicated by $(\cdot)^{\vee}$ which transforms a matrix to a vector such that    
\begin{align}\label{E6}
\boldsymbol{\xi}=(\hat{\boldsymbol{\xi}})^{\vee}=
\begin{pmatrix} 
\hat{\boldsymbol{\omega}} & \mathbf{v}  \\
\mathbf{0}_{1\times 3} & 0  
\end{pmatrix}^{\vee}=
\begin{pmatrix} 
\mathbf{v} \\ 
\boldsymbol{\omega} 
\end{pmatrix}.
\end{align}
In Eq. (\ref{E5-11}), the Lie algebra $\soo$ contains elements of the tangent space at the identity element $\mathbf{I}_{3\times 3}\in \SO$: 
\begin{align}\label{E2}
T_{\mathbf{I}_{3\times 3}}\SO\cong \soo \coloneqq\{\hat{\boldsymbol{\omega}}: \mathbb{R}^{3}\rightarrow \mathbb{R}^{3}|\boldsymbol{\omega}\in\mathbb{R}^3\}.
\end{align}
An element of $\se$ can be mapped to an element of $\SE$ using the exponential map such that as
\begin{align}\label{E6-1}
\exp(\hat{\boldsymbol{\xi}})=
\begin{pmatrix}
    \exp(\hat{\boldsymbol{\omega}}) & \mathcal{T}({\boldsymbol{\omega}})\mathbf{v} \\
    \mathbf{0}_{1\times 3} & 1
\end{pmatrix},
\end{align}
where we have $\mathcal{T}({\boldsymbol{\omega}})=\mathbf{I}_{3\times 3}+\dfrac{1-\cos(|\boldsymbol{\omega}|)}{|\boldsymbol{\omega}|^{2}}\hat{\boldsymbol{\omega}}+\dfrac{|\boldsymbol{\omega}|-\sin(|\boldsymbol{\omega}|)}{|\boldsymbol{\omega}|^{3}}\hat{\boldsymbol{\omega}}^{2}$ and $\exp(\hat{\boldsymbol{\omega}})=\mathbf{I}_{3\times 3}+\dfrac{\sin(|\boldsymbol{\omega}|)}{|\boldsymbol{\omega}|}\hat{\boldsymbol{\omega}}+\dfrac{1-\cos(|\boldsymbol{\omega}|)}{|\boldsymbol{\omega}|^{2}}\hat{\boldsymbol{\omega}}^{2}$. Note that when $\boldsymbol{\omega}=\boldsymbol{0}_{3\times 1}$ we have $\exp(\hat{\boldsymbol{\xi}})=\begin{pmatrix}
    \mathbf{I}_{3\times 3} & \mathbf{v} \\
    \boldsymbol{0}_{1\times 3} & 1
\end{pmatrix}$.
\begin{remark}
    A singularity in the rotation parametrization arises in the formula for $\exp(\hat{\boldsymbol{\xi}})$ when $|\boldsymbol{\omega}|$ approaches $(2n+1)\pi,\; n = 0,1,2,\ldots$. To avoid this singularity, many researchers have used quaternion representations with four parameters. In our work, however, the incremental update procedure for extracting the incremental twist ensures that the rotation vector does not reach these critical values of $|\boldsymbol{\omega}|$.
\end{remark}
The Adjoint operator $\Ad_{\mathbf{g}}: \mathbb{R}^6\rightarrow \mathbb{R}^6$ relates the twists in two coordinate frames with relative transformation $\mathbf{g}\in \SE$:
\begin{align}\label{E6-2}
\Ad_{\mathbf{g}}=
\begin{pmatrix}
    \mathbf{R} & \hat{\mathbf{P}}\mathbf{R} \\
    \boldsymbol{0}_{3\times 3} & \mathbf{R}
\end{pmatrix}.
\end{align}
We have the following properties of the Adjoint operator: (i) $\Ad_{\mathbf{g}}^{-1}=\Ad_{\mathbf{g}^{-1}}$ and  (ii) given $\mathbf{g}_{1}, \mathbf{g}_{2}\in \SE$, it follows that  $\Ad_{\mathbf{g}_{1}\mathbf{g}_{2}}=\Ad_{\mathbf{g}_{1}}\Ad_{\mathbf{g}_{2}}$. 
Furthermore $\forall \boldsymbol{\xi} \in \mathbb{R}^{6}$, the adjoint representation of $\se$, $\ad_{\boldsymbol{\xi}}:\mathbb{R}^{6} \rightarrow \mathbb{R}^{6}$, is defined based on the Lie bracket operator, i.e., $\forall \boldsymbol{\eta} \in \mathbb{R}^{6}$ we have $\ad_{\boldsymbol{\xi}}\boldsymbol{\eta}=[\hat{\boldsymbol{\xi}},\hat{\boldsymbol{\eta}}]^{\vee}=(\hat{\boldsymbol{\xi}}\hat{\boldsymbol{\eta}}-\hat{\boldsymbol{\eta}}\hat{\boldsymbol{\xi}})^{\vee}$, where in matrix form
\begin{align}\label{E6-2s}
\ad_{\boldsymbol{\xi}}=
\begin{pmatrix}
    \hat{\boldsymbol{\omega}} & \hat{\mathbf{v}} \\
    \mathbf{0}_{3\times 3} &  \hat{\boldsymbol{\omega}}
\end{pmatrix}.
\end{align}
Dual of the adjoint operator $\ad_{\boldsymbol{\xi}}^{*}:\mathbb{R}^{6*}\rightarrow \mathbb{R}^{6*}$ is the transpose of this matrix representation. We define the map $\tilde{\ad}_{\boldsymbol{\gamma}}:\mathbb{R}^{6}\rightarrow \mathbb{R}^{6*}$ for every $\boldsymbol{\gamma}=(\mathbf{n}^{T},\mathbf{m}^{T})^T\in \mathbb{R}^{6*}$ such that $\tilde{\ad}_{\boldsymbol{\gamma}}\boldsymbol{\xi}=\ad^{*}_{\boldsymbol{\xi}}\boldsymbol{\gamma}$, which takes the following matrix form:  
\begin{align}\label{E6-2ss}
\tilde{\ad}_{\boldsymbol{\gamma}}=
\begin{pmatrix}
\mathbf{0}_{3\times 3} &  \hat{\mathbf{n}}\\
    \hat{\mathbf{n}} & \hat{\mathbf{m}}
\end{pmatrix}.
\end{align}

\section{Cosserat Shell Media}\label{sec:shell}
\subsection{Kinematics}
A 2D embedded submanifold of the 3D (physical) affine Euclidean space $\mathcal{S}$ defines a simple Cosserat shell body, where to each point of this submanifold we attach a set of right-handed orthonormal triad rigid directors. Therefore, the state of every material point on the Cosserat shell can be represented by a member of  $\SE$ that includes the position of the point and the rotation of the rigid triad attached to it, with respect to an inertially fixed frame $\mathbf{E}$ in $\mathcal{S}$. The right-handed frame $\mathbf{E}:=\{\mathbf{O},\mathbf{E}_{1},\mathbf{E}_{2},\mathbf{E}_{3}\}$ is defined by a point $\mathbf{O}\in\mathcal{S}$ and three orthonormal vectors $\mathbf{E}_I$ ($I=1,2,3$) attached at $\mathbf{O}$.
A \textit{configuration} of a Cosserat shell body is then identified by a smooth 2D embedding 
\begin{align}\label{E7-1}
\mathbf{g}:\mathcal{A}\hookrightarrow \SE,
\end{align}
where we assume there is a globally smooth parametrization of the shell by the connected open subset $\mathcal{A}\subset\mathbb{R}^{2}$. The set of all configurations of the shell body is denoted by
\begin{align}\label{E7}
\boldsymbol{\mathcal{Q}}:=\{\mathbf{g}\in\mathcal{C}^{\infty}(\mathcal{A}  \hookrightarrow \SE)|\,\mathcal{A}\subset\mathbb{R}^{2}\},
\end{align}
which is an infinite dimensional configuration space.
A series of smooth configuration changes of the Cosserat shell is then a $t$-parameterized family of embeddings in $\boldsymbol{\mathcal{Q}}$, which is a smooth mapping $t \mapsto \mathbf{g}_{t}\in \boldsymbol{\mathcal{Q}}$ for $t\in\mathbb{R}^{\geq 0}$. We specifically name the map $\mathbf{g}_{0}:\mathcal{A}\hookrightarrow\SE$ \emph{reference configuration}, which is assumed to be the stress-free shell configuration and name the map $\mathbf{g}_{t}:\mathcal{A}\hookrightarrow\SE$ \emph{current configuration}. Based on Eq.(\ref{E4}), we can write
\begin{align}\label{E10}
\mathbf{g}_{0}(\xi^{1},\xi^{2})=
\begin{pmatrix} 
\mathbf{R}_{0}(\xi^{1},\xi^{2}) & \boldsymbol{\varphi}_{0}(\xi^{1},\xi^{2})  \\
\mathbf{0}_{1\times 3} & 1  
\end{pmatrix}\in\SE,
\end{align}
and
\begin{align}\label{E11}
\mathbf{g}_{t}(\xi^{1},\xi^{2})=
\begin{pmatrix} 
\mathbf{R}_{t}(\xi^{1},\xi^{2}) & \boldsymbol{\varphi}_{t}(\xi^{1},\xi^{2})  \\
\mathbf{0}_{1\times 3} & 1  
\end{pmatrix}\in\SE,
\end{align}
where $(\xi^1,\xi^2)\in\A$ and we consider the rotation maps $\mathbf{R}_{0}:\mathcal{A}\rightarrow \SO$, $\mathbf{R}_{t}:\mathcal{A}\rightarrow \SO$ and the translation maps $\boldsymbol{\varphi}_{0}:\mathcal{A}\rightarrow \mathbb{R}^{3}$,  $\boldsymbol{\varphi}_{t}:\mathcal{A}\rightarrow \mathbb{R}^{3}$ to define the embeddings.
\begin{figure}[h]
\includegraphics[scale=0.4]{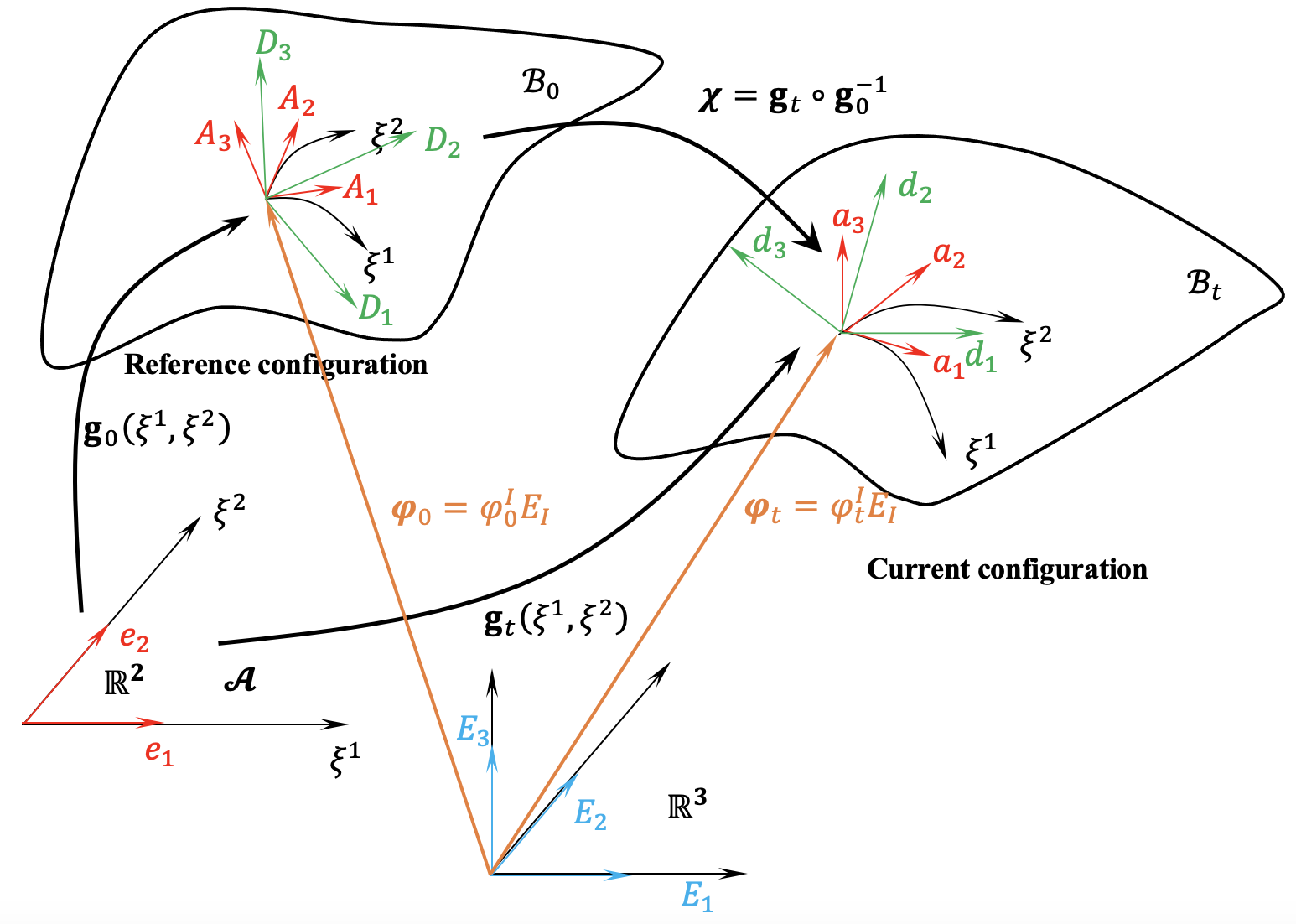}
\centering
\caption{Cosserat shell configuration spaces (Note: The term $\mathbf{g_0^{-1}}$ in the definition of $\boldsymbol{\chi}$ is the inverse of function $\mathbf{g_0}\colon\mathcal{A}\rightarrow\SE$ and not the group inverse)}
\label{fig1}
\end{figure}

 The geometry of a Cosserat shell in $\mathcal{S}$, i.e., its \textit{physical state}, in the reference and current configurations is respectively captured by the $2$D embedded submanifolds $\mathcal{B}_0:=\mathbf{g}_{0}(\mathcal{A})\subset\SE$ and  $\mathcal{B}_t:=\mathbf{g}_{t}(\mathcal{A})\subset\SE$, the image of the maps $\mathbf{g}_{0}$ and $\mathbf{g}_{t}$, depicted in Figure \ref{fig1}.  In this figure, the orthonormal triad directors in reference and current physical states are presented by $D_{I}$ and $d_{I}~(I=1,2,3)$, respectively.
 Additionally, the positions of material particles in the reference and current physical states are  $\boldsymbol{\varphi}_{0}=\sum_{I=1}^{3}\varphi^{I}_{0}(\xi^{1},\xi^{2})\boldsymbol{E}_{I}$ and $\boldsymbol{\varphi}_{t}=\sum_{I=1}^{3}\varphi^{I}_{t}(\xi^{1},\xi^{2})\boldsymbol{E}_{I}$, with corresponding derivatives denoted by $\boldsymbol{A}_{\alpha}\coloneqq\dfrac{\partial\boldsymbol{\varphi}_{0}}{\partial \xi^{\alpha}}\in\mathbb{R}^3$ and $\boldsymbol{a}_{\alpha}\coloneqq\dfrac{\partial\boldsymbol{\varphi}_{t}}{\partial \xi^{\alpha}}\in\mathbb{R}^3$ $(\alpha =1,2)$, indicating the convected basis of the shell surface.

 For a reference shell configuration $\mathbf{g}_{0}\in\mathcal{Q}$ and a current configuration $\mathbf{g}_{t}\in\mathcal{Q}$, the deformation is the diffeomorphism 
 \begin{align}\label{E12}
\boldsymbol{\chi}\colon &\mathcal{B}_{0}\xrightarrow[]{\;\;\cong\;\;} \mathcal{B}_{t}\\ \nonumber
&\mathbf{g}_0(\xi^1,\xi^2)\mapsto\mathbf{g}_t(\xi^1,\xi^2)
\end{align}
 that is defined by the commutative diagram in Figure \ref{fig2} for some $(\xi^1,\xi^2)\in\mathcal{A}$. Note that if $t=0$, then $\boldsymbol{\chi}$ is the identity map on $\mathcal{B}_0$.
The Jacobian of $\boldsymbol{\chi}$ as a function on $\mathcal{A}\subset\mathbb{R}^2$ is called the \emph{deformation gradient} and defined by 
 \begin{align}\label{E12-1}
\mathbf{F}(\xi^{1},\xi^{2}):=T_{\mathbf{g}_0(\xi^1,\xi^2)}\boldsymbol{\chi}: T_{\mathbf{g}_{0}(\xi^{1},\xi^{2})}\mathcal{B}_{0}\rightarrow T_{\mathbf{g}_{t}(\xi^{1},\xi^{2})}\mathcal{B}_{t}.
\end{align}
Since the shell is assumed to be globally parameterized by $\mathcal{A}$  for any arbitrary $t\in\mathbb{R}^{\geq 0}$, the set of vector fields
$X_t=\left\{\left.\frac{\partial \mathbf{g}_t }{\partial \xi^\alpha}\subset T\mathcal{B}_t\right| \alpha=1,2\right\}$ can be considered as a basis that pointwise spans the tangent spaces $T_{\mathbf{g}_t(\xi^1,\xi^2)}\mathcal{B}_t$ for all $(\xi^1,\xi^2)\in\mathcal{A}$. The dual basis $X_t^*=\left\{\left.\big(\frac{\partial \mathbf{g}_t }{\partial \xi^\alpha}\big)^*\subset T^*\mathcal{B}_t\right| \alpha=1,2\right\}$ is then abstractly defined by $\left<\big(\frac{\partial \mathbf{g}_t }{\partial \xi^\alpha}\big)^*|\frac{\partial \mathbf{g}_t }{\partial \xi^\beta}\right>=\delta_{\alpha}^{\beta}$, where $\delta_{\alpha}^{\beta}$ is the kronecker delta.
In this basis, the deformation gradient can be expressed as a two-point $(1,1)$ tensor:
\begin{align}\label{E12-4}
\mathbf{F}=\sum_{\alpha=1}^{2}\dfrac{\partial \mathbf{g}_{t}}{\partial \xi^{\alpha}}\otimes \Big(\dfrac{\partial \mathbf{g}_{0}}{\partial \xi^{\alpha}}\Big)^{*}.
\end{align}
This form of the deformation gradient is inappropriate for numerical implementation due to the absence of a matrix representation. Hence, in the following, we define a new form of the deformation gradient using the Lie-algebraic representation of infinitesimal deformations.

Given a shell's physical state for an arbitrary parameter $t\in\mathbb{R}^{\geq 0}$ as the 2D embedded submanifold $\mathcal{B}_t\subset\SE$, the set $X_t$ spanning the tangent bundle $T\mathcal{B}_t$ can be identified by $\se$-valued sections over $\mathcal{B}_t$:
\begin{align}\label{eq:zeta}
    \mathcal{X}_t:=\left\{\left.\boldsymbol{\zeta}_{t\alpha}(\xi^1,\xi^2)=\big(\mathbf{g}_t^{-1}\frac{\partial \mathbf{g}_t }{\partial \xi^\alpha}\big)^\vee\in \mathbb{R}^6 \cong \se\right| (\xi^1,\xi^2)\in\A, \alpha=1,2\right\}.
\end{align}
This set of sections describe the material deformation field in the local coordinate frame associated with the material point following deformation
 \footnote{This set can be considered as collection of sections of the trivial bundle $\mathcal{B}_t \times \se$ resulted from left trivialization of $T\SE$, which are parameterized by the shell parameters $(\xi^1,\xi^2)$.}. 
We now define the novel notion of \textit{local deformation gradient} through the following relation:
\begin{align}\label{E12-3}
\mathbf{F}_{e}(\xi^1,\xi^2)({\boldsymbol{\zeta}}):=\big(\mathbf{g}_{t}^{-1}\mathbf{F}(\mathbf{g}_{0}\hat{\boldsymbol{\zeta}})\big)^\vee\in\mathbb{R}^6\cong \se,
\end{align}
for any arbitrary \textit{local infinitesimal deformation} ${\boldsymbol{\zeta}}\in\spn \mathcal{X}_t(\xi^1,\xi^2)$. Equivalently, the local deformation gradient can be expressed as a two-point $(1,1)$ tensor on the vector space $\mathbb{R}^6\cong\se$: 
\begin{align}\label{loc-def-grad}
\mathbf{F}_e=\sum_{\alpha=1}^{2}\boldsymbol{\zeta}_{t\alpha }\otimes \boldsymbol{\zeta}_{0\alpha }^{*},
\end{align}
such that the dual vectors $\mathcal{X}^*_0=\left\{\left.\boldsymbol{\zeta}_{0\alpha}^*\subset \mathbb{R}^{6*} \cong \mathfrak{se}^*(3)\right| \alpha=1,2\right\}$ are determined via $\left<\boldsymbol{\zeta}_{0\alpha}^*,\boldsymbol{\zeta}_{0\beta}\right>=\delta_{\alpha}^{\beta}$. To compute the dual vectors we consider the natural basis of $\mathbb{R}^6$ (induces a basis for $\se$ via the hat operator) denoted by $\left\{\mathbf{e}_i\in\mathbb{R}^6|i=1,\cdots,6\right\}$ with the dual basis being $\left\{\mathbf{e}_i^*\in\mathbb{R}^{6*}|i=1,\cdots,6\right\}$. In these coordinates, the dual and transpose of a vector coincide and the dual basis $\mathcal{X}_0^*$ can be calculated by defining the matrix 
$\mathbb{X}_0\coloneqq\begin{pmatrix}
        \boldsymbol{\zeta}_{01} & \boldsymbol{\zeta}_{02}
    \end{pmatrix}\in\mathbb{R}^{6\times 2}$:\footnote{Note that with some abuse of notation, we use the same symbols to denote vectors and their duals in the natural basis and its corresponding dual basis of $\mathbb{R}^6$.}
    \begin{align*}
        \mathbb{X}_0^*\coloneqq\begin{pmatrix}
        \boldsymbol{\zeta}_{01}^* \\ \boldsymbol{\zeta}_{02}^*
    \end{pmatrix}=\big(\mathbb{X}_0^T\mathbb{X}_0\big)^{-1}\mathbb{X}_0^T\in\mathbb{R}^{2\times 6}.
    \end{align*}
    That is the rows of the Moore-Penrose pseudo-inverse of $\mathbb{X}_0$ represent the basis elements in $\mathcal{X}_0^*$. Therefore, in matrix form, the local deformation gradient can be represnted by 
    \begin{align}\label{E1213}
        \mathbf{F}_{e}=\mathbb{X}_t\mathbb{X}^*_0\in\mathbb{R}^{6\times 6}.
    \end{align}
It is important to note that the map $\mathbf{F}_{e}$ takes material deformation in the local coordinate frame of the reference configuration to the corresponding deformation in the current configuration.
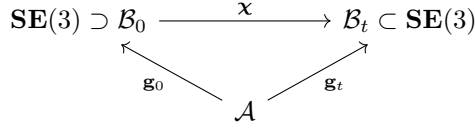
\begin{figure}
\centering
\begin{tikzcd}
	\SE\supset\mathcal{B}_0 & {} & \mathcal{B}_t\subset\SE \\
	& {\mathcal{A}}
	\arrow["{\boldsymbol{\chi}}", from=1-1, to=1-3]
	\arrow["{\mathbf{g}_{t}}"', from=2-2, to=1-3]
	\arrow["{\mathbf{g}_{0}}", from=2-2, to=1-1]
\end{tikzcd}
\caption{Commutative diagram}
\label{fig2}
\end{figure}
In the absence of deformation, Eq. (\ref{E1213}) simplifies to
\begin{align}\label{E1214}
\mathbf{I}_{e}=\mathbb{X}_0\mathbb{X}^*_0\in\mathbb{R}^{6\times 6},
\end{align}
and accordingly, the Cosserat shell strain tensor can be defined as
\begin{align}\label{E16}
\boldsymbol{\epsilon}=\mathbf{F}_{e}-\mathbf{I}_{e}=\boldsymbol{\mathcal{E}}\mathbb{X}^*_0,
\end{align}
where we have $\boldsymbol{\mathcal{E}}\coloneqq\mathbb{X}_t-\mathbb{X}_0$.

\subsection{Cosserat Shell Balance Equations}
In this section, we examine a Lagrangian function $L: \boldsymbol{\mathcal{Q}}\rightarrow \mathbb{R}$ to derive the quasi-static balance equations of a Cosserat shell structure. This function includes the internal energy of the shell and for a given reference configuration  $\g_0\in \boldsymbol{\mathcal{Q}}$  takes the following form:
\begin{align}\label{E14}
L(\g_t)=\int_{\mathcal{A}_{0}}L_{0}\Big(\mathbf{g}_{t}, \dfrac{\partial \mathbf{g}_{t}}{\partial \xi^{1}}, \dfrac{\partial \mathbf{g}_{t}}{\partial \xi^{2}}\Big)\bar{j}_{0}d\mathcal{A}.
\end{align}
Here, the Lagrangian per unit area of a Cosserat shell in the reference configuration is defined as the function $L_{0}$ dependent on the embedding $\g_t$ and its derivatives defining vector fields over $\mathcal{B}_t$. Additionally,
$\bar{j}_{0}d\mathcal{A}$ is the surface element at an arbitrary point of the shell in the reference configuration, while we have $d\mathcal{A}=d\xi^{1}d\xi^{2}$ and  $\bar{j}_{0}=\|\boldsymbol{A}_{1}\times \boldsymbol{A}_{2}\|$, following shell mid-surface base vectors. The principle of material frame indifference indicates that the shell's Lagrangian per unit area must satisfy \cite{boyer2017poincare}
\begin{align}\label{E16-1}
L_{0}\Big(\mathbf{g}_{t}, \dfrac{\partial \mathbf{g}_{t}}{\partial \xi^{1}}, \dfrac{\partial \mathbf{g}_{t}}{\partial \xi^{2}}\Big)=L_{0}\Big(\mathbf{h} \mathbf{g}_{t}, \mathbf{h} \dfrac{\partial \mathbf{g}_{t}}{\partial \xi^{1}}, \mathbf{h} \dfrac{\partial \mathbf{g}_{t}}{\partial \xi^{2}}\Big),
\end{align} 
for all $\mathbf{h}\in\SE$, i.e., $L_0$ is left invariant under the rigid transformations of the ambient space in which our shell embeds, i.e., the observer. Now considering $\mathbf{h}=\mathbf{g}_{t}^{-1}$, we have
\begin{align}\label{E17}
L_{0}\Big(\mathbf{I}_{4\times 4},\mathbf{g}_{t}^{-1} \dfrac{\partial \mathbf{g}_{t}}{\partial \xi^{1}}, \mathbf{g}_{t}^{-1} \dfrac{\partial \mathbf{g}_{t}}{\partial \xi^{2}}\Big)=:\ell_{0}(\boldsymbol{\zeta}_{t1}, \boldsymbol{\zeta}_{t2}),
\end{align}
defining the function $\ell_0\colon\mathbb{R}^6\times\mathbb{R}^6\rightarrow\mathbb{R}$ on the Lie algebra level. Hence, the principle of the material frame indifference reduces the Lagrangian $L_0$ to the Lagrangian $\ell_0$ on $\mathbb{R}^6\times\mathbb{R}^6\cong\se\times\se$.
It is worth noting that the Lie algebraic formulation of the deformation gradient in equation \eqref{E1213} inherently satisfies the principle of material frame indifference, i.e., the deformation gradient transforms appropriately under a superposed rigid body motion.

 Let $\mathbf{g}_t\in \mathcal{Q}$ be a smooth 2D embedding (surface) in $\SE$ defined on the compact set $\mathcal{A}_0\subset\mathbb{R}^2$ with the boundary $\partial\mathcal{A}_0$. A variation of $\mathbf{g}_t$ with fixed boundary is a smooth map $\breve\g_t: \A_0\times\R\rightarrow \SE$ that satisfies  the  conditions $\breve \g_t(\xi^1,\xi^2,0)  =\g_t(\xi^1,\xi^2)$ and $\frac{\partial}{\partial\varepsilon}\Big(\left.\breve\g_t\right|_{\partial\A_0}\Big)=\0_{4\times 4}$.  This  variation defines the vector field
\begin{equation}\label{eq:var}
    \delta \g_t(\xi^1,\xi^2)=\left.\frac{\partial \breve\g_t(\xi^1,\xi^2,\varepsilon)}{\partial \varepsilon} \right|_{\varepsilon=0}\in T_{\g_t(\xi^1,\xi^2)}\SE
\end{equation}
on the surface $\mathcal{B}_t=\g(\A_0)$, such that $ \left.\delta \g_t\right|_{\partial\A_0}=\0_{4\times 4}$. The operator $\delta$ always refers to variation of an entity, i.e., its composition with $\breve\g_t$ and taking the derivative with respect to $\varepsilon$ on the surface $\mathcal{B}_t$.
Then, the balance equations are obtained using the principle of virtual work, that is defined based on the following variational principle for the variations of the type $\delta\g_t$ in Eq. \eqref{eq:var}:
\begin{align}\label{E18}
\delta L+\delta W_{ext}=\delta\int_{\A_0} L_0\Big(\mathbf{g}_{t}, \dfrac{\partial \mathbf{g}_{t}}{\partial \xi^{1}}, \dfrac{\partial \mathbf{g}_{t}}{\partial \xi^{2}}\Big) \bar{j}_{0}d\A+\delta W_{ext}=0,
\end{align} 
where $W_{ext}$ is the work of external wrenches applied to the shell. At this point, we do not specify the nature of these wrenches.
We define the variation
\begin{align}\label{E21}
\hat{\boldsymbol{\kappa}}=\mathbf{g}_{t}^{-1}\delta \mathbf{g}_{t}\in\se,
\end{align}
as the left translation of the variation of the surface $\g_t$ to the lie algebra $\se$, with the condition $\left.\boldsymbol{\kappa}\right|_{\partial\A_0}=\0_{6\times 1}$.

\begin{lemma}
    A Cosserat shell with Lagrangian $L$ satisfies the variational principle in Eq. \eqref{E18} if and only if $\ell_0$ in Eq. \eqref{E17} satisfies the variational principle:
    \begin{align}\label{E19}
\delta \int_{\mathcal{A}_{0}}\ell_{0}(\boldsymbol{\zeta}_{t1}, \boldsymbol{\zeta}_{t2})\bar{j}_{0}d\mathcal{A}=-\delta W_{ext},
\end{align}
for variations of type ($\alpha=1,2$):
\begin{align}\label{eq:var-zeta}
    \delta \boldsymbol{\zeta}_{t\alpha}=\frac{\partial \boldsymbol{\kappa}}{\partial\xi_\alpha}+\ad_{\boldsymbol{\zeta}_{t\alpha}}\boldsymbol{\kappa}.
\end{align}
\end{lemma}
\begin{proof}
    The detailed proof is presented in Appendix \ref{app:lem1}.
\end{proof}

Based on the definition presented in Eq. (\ref{E21}), the variation of external work can be defined as
\begin{align}\label{E32}
\delta W_{ext}\coloneqq \int_{\mathcal{A}_{0}}\Big<\mathbf{\bar{F}}^{ext},\boldsymbol{\kappa} \Big> \bar{j}_{0}d\mathcal{A}+\int_{\partial \mathcal{A}_{0}}\Big<\mathbf{F}^{ext},\boldsymbol{\kappa} \Big>\bar{j}_{0}ds,
\end{align} 
where $\mathbf{\bar{F}}^{ext}\in (\mathbb{R}^6)^*\cong \mathfrak{se}^*(3)$ is the vector of the body wrenches per unit area applied on the shell surface (body forces and couples or follower loads) and $\mathbf{F}^{ext}\in (\mathbb{R}^6)^*\cong \mathfrak{se}^*(3)$ is the vector of body wrenches per unit length applied on the shell boundaries. 
Substituting Eqs. \eqref{eq:var-zeta} and (\ref{E32}) into (\ref{E19}) leads to the weak form of the balance equations for a Cosserat shell:
\begin{align}\label{E33}\nonumber
\int_{\mathcal{A}_{0}}\sum_{\alpha=1}^2\Big<\dfrac{\partial \ell_{0}}{\partial \boldsymbol{\zeta}_{t\alpha}},\dfrac{\partial \boldsymbol{\kappa}}{\partial \xi^{\alpha}}+\ad_{\boldsymbol{\zeta}_{t\alpha}}\boldsymbol{\kappa}\Big>\bar{j}_{0}d\mathcal{A}=&\\-\int_{\mathcal{A}_{0}}\Big<\mathbf{\bar{F}}^{ext},\boldsymbol{\kappa} \Big> \bar{j}_{0}d\mathcal{A}-&\int_{\partial \mathcal{A}_{0}}\Big<\mathbf{F}^{ext},\boldsymbol{\kappa} \Big>\bar{j}_{0}ds.
\end{align}
Performing integration by part for every term on the left hand side and noting that the variation $\boldsymbol{\kappa}$ vanishes at the boundary:
\begin{align}\label{E34}
\int_{\mathcal{A}_{0}}\!\!\!&-\sum_{\alpha=1}^2\Big<\dfrac{1}{\bar{j}_{0}}\dfrac{\partial}{\partial \xi^{\alpha}}\Big(\bar{j}_{0}\dfrac{\partial \ell_{0}}{\partial \boldsymbol{\zeta}_{t\alpha}}\Big)-\ad^{*}_{\boldsymbol{\zeta}_{t\alpha}}\dfrac{\partial \ell_{0}}{\partial \boldsymbol{\zeta}_{t\alpha}}, \boldsymbol{\kappa}\Big>\bar{j}_{0}d\mathcal{A}\nonumber \\ &=-\int_{\mathcal{A}_{0}}\Big<\mathbf{\bar{F}}^{ext},\boldsymbol{\kappa} \Big> \bar{j}_{0}d\mathcal{A}-\int_{\partial \mathcal{A}_{0}}\Big<\mathbf{F}^{ext}+\dfrac{\partial \ell_{0}}{\partial \boldsymbol{\zeta}_{t1}}\eta_{1}+\dfrac{\partial \ell_{0}}{\partial \boldsymbol{\zeta}_{t2}}\eta_{2},\boldsymbol{\kappa} \Big>\bar{j}_{0}ds.
\end{align}
The strong form of \emph{Euler-Lagrange equation} or \emph{Euler-Poincaré reduction} \cite{poincare1901} of the Cosserat shell, can be written as
\begin{align}\label{E35}
\boldsymbol{\kappa}:\quad \sum_{\alpha=1}^2\dfrac{1}{\bar{j}_{0}}\dfrac{\partial}{\partial \xi^{\alpha}}\Big(\bar{j}_{0}\dfrac{\partial \ell_{0}}{\partial \boldsymbol{\zeta}_{t\alpha}}\Big)-\ad_{\boldsymbol{\zeta}_{t\alpha}}^{*}(\dfrac{\partial \ell_{0}}{\partial \boldsymbol{\zeta}_{t\alpha}})=\mathbf{\bar{F}}^{ext},
\end{align}
with boundary condition
\begin{align}\label{E36}
\sum_{\alpha=1}^2\left(\dfrac{\partial \ell_{0}}{\partial \boldsymbol{\zeta}_{t\alpha}}\eta_{\alpha}\right)_{\partial \mathcal{A}_{0}}=-\mathbf{F}^{ext}\qquad or \qquad \boldsymbol{\kappa}|_{\mathcal{\partial A}_{0}}=0,
\end{align}
Here, the unit outward normal on the boundaries of the region $\A_0$ in $\mathbb{R}^2$ is written as $\eta=(\eta_1,\eta_2)$. It is important to mention that both balances of momentum and moment of momentum in the category of the Cosserat shell are merged into Eq. (\ref{E35}). 

\begin{remark}
    The matrix form of the presented deformation quantities based on Eq. (\ref{E11}) can be written as
\begin{align}\label{E39}
\hat{\boldsymbol{\zeta}}_{t\alpha}=
\begin{pmatrix} 
\mathbf{R}_{t}^{T}\mathbf{R}_{t,\alpha} & \mathbf{R}_{t}^{T}\boldsymbol{\varphi}_{t,\alpha}  \\
0 & 0  
\end{pmatrix}=
\begin{pmatrix} 
\hat{\boldsymbol{\boldsymbol{\Gamma}}}_{\alpha} & \mathbf{C}_{\alpha}  \\
0 & 0 
\end{pmatrix}, \quad
\boldsymbol{\zeta}_{t\alpha}=
\begin{pmatrix} 
\mathbf{C}_{\alpha}  \\
\boldsymbol{\boldsymbol{\Gamma}}_{\alpha}
\end{pmatrix}.
\end{align}
By comparing the results in Eq. (\ref{E39}) with the results presented in \cite{cosserat1909,eringen1976,eringen1999} we conclude that $\mathbf{C}_{\alpha}$ and $\boldsymbol{\boldsymbol{\Gamma}}_{\alpha}$ are a set of Cosserat deformation tensors which are form-invariant under rigid body motions of the spatial frame of reference. $\mathbf{C}_{\alpha}$ is called the \emph{the Cosserat deformation tensor} and $\boldsymbol{\boldsymbol{\Gamma}}_{\alpha}$ is called the \emph{wryness tensor}. Without loss of generality, at the reference configuration if we assume 
\begin{align}\label{E40}
\mathbf{C}_{\alpha}=\boldsymbol{A}_{\alpha},\qquad \boldsymbol{\boldsymbol{\Gamma}}_{\alpha}=0,
\end{align}
then the corresponding momentum quantities can be presented by
\begin{align}\label{E43}
\dfrac{\partial \ell_{0}}{\partial \boldsymbol{\zeta}_{t\alpha}}=
\begin{pmatrix} 
-\mathbf{N}^{\alpha}  \\
-\mathbf{M}^{\alpha}
\end{pmatrix},
\end{align}
where $\mathbf{N}^{\alpha}$ is  the membrane stress and $\mathbf{M}^{\alpha}$ is the bending stress of the shell, in the reference configuration. The presented conjugate terms have the same basis as their kinematics conjugates.
\end{remark}

\begin{remark}
    Following Eq.~(\ref{E39}), due to the arbitrariness of the wryness tensor $\boldsymbol{\Gamma}_{\alpha}$, the component of this vector in the direction of the shell normal $\boldsymbol{n} = \frac{\boldsymbol{a}_1 \times \boldsymbol{a}_2}{\lvert \boldsymbol{a}_1 \times \boldsymbol{a}_2 \rvert}$ corresponds to the drilling rotation. However, this rotational degree of freedom cannot be represented in shell theories that employ only a single attached director.
\end{remark}

\subsection{Internal Energy}
In this study, we consider shells made out of hyper-elastic linear isotropic material with the local constitutive equation defined in the convected basis of the shell:
\begin{align}\label{eq:strain}
    \boldsymbol{S}^\alpha=\sum_{\beta=1}^2\boldsymbol{\mathcal{D}}^{\alpha \beta}\boldsymbol{\mathcal{E}}_\beta\in (\R^6)^*\cong\mathfrak{se}^*(3),\quad\quad\alpha=1,2
\end{align}
where $\boldsymbol{\mathcal{E}}_\beta$ is the $\beta=1,2$ column of $\boldsymbol{\mathcal{E}}$, $\boldsymbol{\mathcal{D}}^{\alpha \beta} \in \mathbb{R}^{6\times 6}$ are the Cosserat shell stiffness matrices and $\boldsymbol{S}^{\alpha}$ are the components of the stress tensor. For linear isotropic constitutive equation, we have the matrices
\begin{align}\label{E45-1}
\boldsymbol{\mathcal{D}}^{\alpha \beta}=\dfrac{Eh}{1-\nu^{2}}
\begin{pmatrix}
    \mathbb{D}_{1}^{\alpha \beta} & \mathbf{0}_{3\times3}\\
    \mathbf{0}_{3\times3} & \dfrac{h^2}{12}\mathbb{D}_{2}^{\alpha \beta}
\end{pmatrix},
\end{align}
where 
\begin{align}\label{E45-2}
\begin{split}
&\mathbb{D}_{1}^{11}=
\begin{pmatrix}
    H^{1111} & H^{1112} & 0\\
    H^{1211} & H^{1212} & 0\\
    0 & 0 & \dfrac{1-\nu}{2}A^{11}
\end{pmatrix}, \quad
\mathbb{D}_{2}^{11}=
\begin{pmatrix}
    H^{1111} & H^{1112} & 0\\
    H^{1211} & H^{1212} & 0\\
    0 & 0 & (1-\nu)A^{11}
\end{pmatrix},\\
&\mathbb{D}_{1}^{12}=
\begin{pmatrix}
    H^{1121} & H^{1122} & 0\\
    H^{1221} & H^{1222} & 0\\
    0 & 0 & \dfrac{1-\nu}{2}A^{12}
\end{pmatrix}, \quad
\mathbb{D}_{2}^{12}=
\begin{pmatrix}
    H^{1121} & H^{1122} & 0\\
    H^{1221} & H^{1222} & 0\\
    0 & 0 & (1-\nu)A^{12}
\end{pmatrix},\\
&\mathbb{D}_{1}^{21}=
\begin{pmatrix}
    H^{2111} & H^{2112} & 0\\
    H^{2211} & H^{2212} & 0\\
    0 & 0 & \dfrac{1-\nu}{2}A^{21}
\end{pmatrix}, \quad
\mathbb{D}_{2}^{21}=
\begin{pmatrix}
    H^{2111} & H^{2112} & 0\\
    H^{2211} & H^{2212} & 0\\
    0 & 0 & (1-\nu)A^{21}
\end{pmatrix},\\
&\mathbb{D}_{1}^{22}=
\begin{pmatrix}
    H^{2121} & H^{2122} & 0\\
    H^{2221} & H^{2222} & 0\\
    0 & 0 & \dfrac{1-\nu}{2}A^{22}
\end{pmatrix}, \quad
\mathbb{D}_{2}^{22}=
\begin{pmatrix}
    H^{2121} & H^{2122} & 0\\
    H^{2221} & H^{2222} & 0\\
    0 & 0 & (1-\nu)A^{22}
\end{pmatrix}.
\end{split}
\end{align}
Here, $H^{\alpha \beta \gamma \rho}=\nu A^{\alpha \beta}A^{\gamma \rho}+(1-\nu)A^{\alpha \gamma}A^{\beta \rho}$ for $\alpha,\beta,\gamma,\rho=1,2$, such that $A^{\alpha \beta}$ denote the elements of the inverse matrix of the $2\times 2$ matrix $[A_{\alpha \beta}]=[\boldsymbol{A}_{\alpha}\cdot \boldsymbol{A}_{\beta}]$, $E$ and $\nu$ are Youngs moduli and Poisson ratio respectively and $h$ is shell thickness \cite{sansour1995co}.
Therefore, the Lagrangian of the homogeneous and isotropic Cosserat shell consisting of the internal potential energy can be written as
\begin{align}\label{E45}
\ell_{0}(\boldsymbol{\zeta}_{t1}, \boldsymbol{\zeta}_{t2})=-\dfrac{1}{2}\sum_{\alpha=1}^2\Big<\boldsymbol{S}^{\alpha},\boldsymbol{\mathcal{E}}_\alpha\Big>,
\end{align}
such that we have $\boldsymbol{S}^{\alpha}=-\frac{\partial\ell_0}{\partial \boldsymbol{\zeta}_{t\alpha}}$.


\subsection{Magnetized Cosserat Shell Body}
Considering a fixed (remanent) magnetization field per unit volume $\mathbf{B}^{r}\in\R^3$ in a continuum under the action of an externally applied field $\mathbf{B}^{a}\in\R^3$, the interaction is a force per unit volume
\begin{align}\label{E46}
\mathbf{f}=\dfrac{1}{\mu_{0}}\mathbf{B}^{r}\cdot \nabla\mathbf{B}^{a}\in\R^3,
\end{align}
and a couple per unit volume
\begin{align}\label{E47}
\mathbf{m}=\dfrac{1}{\mu_{0}}\mathbf{B}^{r}\times \mathbf{B}^{a}\in\R^3,
\end{align}
where $\mu_{0}$ is the permeability of space and $\nabla$ is the spatial gradient operator \cite{brown1966}. Here, the operators $\cdot$ and $\times$ are the usual dot and cross product defined on $\R^3$.

In this paper, a thin Cosserat shell with small-scale hard magnetic particles distributed homogeneously over the shell mid-surface is considered. The particles are fixed to the shell.
Therefore, the shell comes with a remanent magnetization given by the field $\mathbf{B}^{r}_{0}(\xi^1,\xi^2)\in\R^3$ in the inertial coordinate frame per unit surface of the reference configuration. The field is constant in the local material frames. Upon deformation, the remnant magnetization per unit \emph{undeformed surface} is then given by
\begin{align}\label{E48}
\mathbf{B}^{r}_{t}=\mathbf{R}_{t}\mathbf{R}_{0}^T\mathbf{B}^{r}_{0}\in\R^3.
\end{align}
Assuming, for specificity, that the external magnetic field $\mathbf{B}^{a}$ is spatially constant, the magnetic body force vanishes, while the magnetic body couple per unit undeformed surface in the inertial frame is given by 
\begin{align}\label{E49}
\mathbf{m}_t=\dfrac{1}{\mu_{0}}\mathbf{B}^{r}_{t}\times \mathbf{B}^{a}\in\R^3.
\end{align}
The external virtual work of the magnetic couple for a reference configuration is then obtained as
\begin{align}\label{E50}
G_{m}(\mathbf{g}_{t},{\boldsymbol{\kappa}})=\int_{\mathcal{A}_0}\Big<\mathbf{R}_{t}^T\mathbf{m}_t,\boldsymbol{\kappa}_{R} \Big>\bar{j}_{0}d\mathcal{A},
\end{align} 
where ${\boldsymbol{\kappa}}_{R}\in\R^3$ is defined by the following equation 
\begin{align}\label{E51}
\hat{\boldsymbol{\kappa}}=
\begin{pmatrix} 
\hat{\boldsymbol{\kappa}}_{R} & \boldsymbol{\kappa}_{\varphi}  \\
0 & 0  
\end{pmatrix}=
\begin{pmatrix} 
\mathbf{R}_t^{T}\delta \mathbf{R}_t & \mathbf{R}^{T}_t\delta \boldsymbol{\varphi}_t  \\
0 & 0 
\end{pmatrix}.
\end{align}
\section{Linearization of the Weak Form of Balance Equations}\label{sec:lin}
Based on Eq. (\ref{E33}), the Cosserat shell's weak form of balance equation can be written as
\begin{align}\label{E61}
G_{tot}(\mathbf{g}_{t},{\boldsymbol{\kappa}})\coloneqq G_{int}(\mathbf{g}_{t},{\boldsymbol{\kappa}})+\delta W_{ext}(\mathbf{g}_{t},{\boldsymbol{\kappa}})+G_{m}(\mathbf{g}_{t},{\boldsymbol{\kappa}})=0,
\end{align}
where the functional
\begin{align}\label{E62}
G_{int}(\mathbf{g}_{t},{\boldsymbol{\kappa}})=\int_{\mathcal{A}_{0}}\Big(\sum_{\alpha=1}^2\Big<\boldsymbol{S}^{\alpha},\dfrac{\partial \boldsymbol{\kappa}}{\partial \xi^{\alpha}}+\ad_{\boldsymbol{\zeta}_{t\alpha}} \boldsymbol{\kappa}\Big>\Big)\bar{j}_{0}d\mathcal{A},
\end{align}
and $\delta W_{ext}$ and $G_m$ are found from Eqs. \eqref{E32} and \eqref{E50}, receptively.
The weak form presented constitutes a nonlinear equilibrium equation for a Cosserat shell. For numerical implementation via the finite element method, it is necessary to compute the linearization of this weak form about a given configuration $\bar{\mathbf{g}}_t(\xi^1,\xi^2)$ in an arbitrary direction ${\boldsymbol{\eta}}(\xi^1,\xi^2)\in \R^6\cong \se$, specified by a field of Lie algebra elements on $\mathcal{B}_t$. 

To this end, we introduce the perturbed configuration of the Cosserat shell as
\begin{align}\label{E63}
\mathbf{g}_{\varepsilon}=\bar{\mathbf{g}}_t\exp(\varepsilon \hat{\boldsymbol{\eta}}),
\end{align}
where $\varepsilon\in\R$ is the scalar perturbation parameter. Accordingly, for any differentiable function $f\colon\mathcal{Q}\rightarrow\R$, we define the linearization operator $\mathcal{L[\cdot]}$ evaluated at $\bar\g_t\in\mathcal{Q}$ in the direction of $\boldsymbol{\eta}\in\R^6$ as
\begin{align}
    \mathcal{L}[f](\mathbf{g}_t)\coloneqq f(\mathbf{g}_t)+\left.\frac{\partial}{\partial \varepsilon}\right|_{\varepsilon=0}f(\mathbf{g}_\varepsilon).
\end{align}
In the following, we derive the linearized form of each term in Eq. \eqref{E61}, based on this formulation.


\subsection{Linearization of Internal Energy}

The linearization of the functional $G_{int}$ about the configuration $\bar{\mathbf{g}}_{t}$ is then expressed as 
\begin{align}\label{E65}
\mathcal{L}[G_{int}](\bar{\mathbf{g}}_{t},{\boldsymbol{\kappa}})=G_{int}(\bar{\mathbf{g}}_{t},\boldsymbol{\kappa})+\left.\dfrac{\partial}{\partial \varepsilon}\right|_{\varepsilon=0}
G_{int}(\mathbf{g}_{\varepsilon},\boldsymbol{\kappa}).
\end{align}
In this formulation, the first term on the right-hand side corresponds to the residual (unbalanced) internal forces, while the second term yields the directional derivative of the internal virtual work required for assembling the tangent stiffness matrix.
For brevity in future derivations, we define the following operator:  
\begin{align}\label{E66}
\mathbf{K}_{\alpha}=\Big[\dfrac{\partial}{\partial \xi^{\alpha}}\mathbf{I}_{6\times6}+\ad_{\boldsymbol{\zeta}_{t\alpha}}\Big].
\end{align}

\begin{proposition}\label{prop:1}
The second term of Eq. (\ref{E65}) can be calculated by
\begin{align}\label{E67}
\left.\dfrac{\partial}{\partial \varepsilon}\right|_{\varepsilon=0}
\!\!\!\!G_{int}(\mathbf{g}_{\varepsilon},\boldsymbol{\kappa})=\int_{\mathcal{A}_{0}}\sum_{\alpha=1}^2\Big(\Big<\sum_{\beta=1}^2\boldsymbol{\mathcal{D}}^{\alpha \beta}\bar{\mathbf{K}}_{\beta}\boldsymbol{\eta},\bar{\mathbf{K}}_{\alpha}\boldsymbol{\kappa}\Big>-\Big<\tilde{\ad}_{\bar{\boldsymbol{S}}^{\alpha}}\boldsymbol{\kappa},\bar{\mathbf{K}}_{\alpha}\boldsymbol{\eta}\Big>\Big)\bar{j}_{0}d\mathcal{A},
\end{align}
where the bar means evaluation of states and operators at $\g_t=\bar\g_t$ and $\boldsymbol{\zeta}_{t\alpha}=\bar{\boldsymbol{\zeta}}_{t\alpha}\coloneqq\bar\g_t^{-1}\dfrac{\partial\bar\g_t}{\partial\xi^\alpha}$. Further, the term $\tilde{\ad}_{\bar{\boldsymbol{S}}_{\alpha}}$ is calculated based on Eq. \eqref{E6-2ss}. 

\end{proposition}
\begin{proof}
The detailed proof is presented in Appendix \ref{app:pro1}.
\end{proof}
\begin{remark}
The first term in Eq. (\ref{E67}) is called \emph{material} part of the tangent stiffness matrix, and the second term is considered as \emph{geometric} part of the tangent stiffness matrix. In Simo's terminology \cite{simo1986rod, simo1992sym}, the geometric stiffness matrix is a non-symmetric tangent stiffness matrix away from the shell's equilibrium, i.e., it has a nonzero skew-symmetric part. To prove this statement see Appendix \ref{Remark2}.  

\end{remark}
\subsection{Linearization of External Magnetic Energy}

To account for the contribution of magnetic couple effects in the numerical modeling of shell structures, we derive the linearized form of the virtual external magnetic work associated with a hard-magnetic soft Cosserat shell. This linearization is performed at a given configuration $\bar\g_t$ within the shell's configuration space, along a perturbation direction $\boldsymbol{\eta}=\begin{bmatrix}
    \boldsymbol{\eta}_\varphi\\ \boldsymbol{\eta}_R
\end{bmatrix}\in\R^6\cong \se$. Following the formulation provided in Eq. (\ref{E50}), the linearized virtual magnetic energy can be expressed as:
\begin{align}\label{E79}
\mathcal{L}[G_{m}](\bar{\mathbf{g}}_{t},{\boldsymbol{\kappa}})=G_{m}(\bar{\mathbf{g}}_{t},\boldsymbol{\kappa})+\left.\dfrac{\partial}{\partial \varepsilon}\right|_{\varepsilon=0}
G_{m}(\mathbf{g}_{\varepsilon},\boldsymbol{\kappa}).
\end{align}
By substituting Eqs. (\ref{E48}) and (\ref{E49}) into Eq. (\ref{E79}) and noting that the reference magnetic field $\mathbf{R}_0^T\mathbf{B}^{r}_{0}$ remains fixed, we obtain:
\begin{align}\label{E80}\nonumber
\mathcal{L}[G_{m}](\bar{\mathbf{g}}_{t},{\boldsymbol{\kappa}})&=\int_{\mathcal{A}_0}\Big<\dfrac{1}{\mu_{0}}(\mathbf{R}_0^T\mathbf{B}^{r}_{0})\times (\bar{\mathbf{R}}^T_t \mathbf{B}^{a}), \boldsymbol{\kappa}_{R}\Big>\bar{j}_0d\mathcal{A}
\\&+\int_{\mathcal{A}_0}\Big<\dfrac{1}{\mu_{0}}(\mathbf{R}_0^T\mathbf{B}^{r}_{0})^\wedge(\bar{\mathbf{R}}_t^T\mathbf{B}^{a})^\wedge \boldsymbol{\eta}_{R}, \boldsymbol{\kappa}_{R}\Big>\bar{j}_0d\mathcal{A}.
\end{align}
This expression captures both the zeroth-order virtual magnetic work and its first-order variation, which is essential for constructing the tangent operators used in the finite element formulation of magnetoelastic shell dynamic.

\section{Finite Element Formulation}\label{sec:fem}
In this section, we present the geometrically exact Lagrangian finite element formulation derived from the previously established weak form of the equilibrium equations under static conditions. The formulation leverages the proposed configuration space of the Cosserat shell, enabling interpolation on the Lie group $\SE$. In particular, we employ an isoparametric four-node finite element defined over the tangent bundle $T\SE$. This geometrically consistent approach to discretizing the Cosserat shell kinematics provides a robust framework that effectively mitigates shear-locking artifacts often encountered in conventional shell formulations. Further details regarding the discretization scheme and the update algorithm will be addressed in the following sections.

\subsection{Cosserat Shell Discretization}
Within the finite element framework, the natural coordinates $(\xi^{1},\xi^{2})\in\A$ describing the mid-surface of the Cosserat shell are mapped onto the reference unit square via bilinear interpolation functions. This yields the following isoparametric relations:
\begin{align}\label{E70-1}
\xi^{1}_{e} =\sum_{i=1}^{4}N^{i}(x,y)\xi_{i}^{1e}, \qquad \xi^{2}_{e} =\sum_{i=1}^{4}N^{i}(x,y)\xi_{i}^{2e},
\end{align}
where $\xi^{1}_{e}$ and $\xi^{2}_{e}$ denote the natural coordinates of a material point within the element $e$, and $\xi_{i}^{1e}$ and $\xi_{i}^{2e}$ are the corresponding nodal coordinates. The bilinear interpolation functions $N^{i}(x,y)$ are defined over the square $(x,y)\in[-1,1]\times [-1,1]$ as:
\begin{align}\label{E70-2}
N^{i}(x,y)=\dfrac{1}{4}(1+xx_{i})(1+yy_{i}), \quad (i=1,2,3,4),
\end{align}
with nodal positions $(x_{i},y_{i})\in\{(-1,-1), (1,-1), (1,1), (-1,1)\}$. The interpolation of the mid-surface quantities, namely the virtual displacement fields $\boldsymbol{\kappa}$ and $\boldsymbol{\eta}$ for variation and linearization, follows the standard isoparametric procedure, leading to: 
\begin{align}\label{E70-3}
\boldsymbol{\eta}_{e}=\sum_{i=1}^{4}N^{i}(x,y)\boldsymbol{\eta}^{e}_{i},\qquad \boldsymbol{\kappa}_{e}=\sum_{i=1}^{4}N^{i}(x,y)\boldsymbol{\kappa}^{e}_{i}
\end{align}
where $\boldsymbol{\eta}^{e}_{i}$ and $\boldsymbol{\kappa}^{e}_{i}$ represent the nodal incremental twist fields associated with element $e$. 

\subsection{Out-of-Balance Forces Computation and The Element Tangent Stiffness Matrices}
In this section, we derive the discrete element-level expressions for internal and external forces, along with the corresponding tangent stiffness matrices, using the proposed four-node isoparametric shell element. The out-of-balance internal force at the element $e$ is obtained by substituting Eq. (\ref{E70-3}) into the internal virtual work expression $G_{{int}}(\bar{\mathbf{g}}_{t},\boldsymbol{\kappa})$, yielding:
\begin{align}\label{E71}
\left.G_{{int}}\right|_e=&\int_{\mathcal{A}_{e0}}\sum_{\alpha=1}^{2}\Big<\bar{\boldsymbol{S}}^{\alpha}_{e},\bar{\mathbf{K}}_{e\alpha}\boldsymbol{\kappa}_{e}\Big>\bar{j}_{e0}d\mathcal{A}_{e}= \int_{\mathcal{A}_{e0}}\sum_{i=1}^{4}\sum_{\alpha=1}^{2}\Big<\bar{\boldsymbol{S}}^{\alpha}_{e},\bar{\mathbf{K}}_{e\alpha}^{i}\boldsymbol{\kappa}^{e}_{i}\Big>\bar{j}_{e0}d\mathcal{A}_{e}\nonumber \\&= \int_{\mathcal{A}_{e0}}\sum_{i=1}^{4}\sum_{\alpha=1}^{2}\Big<(\bar{\mathbf{K}}_{e\alpha}^{i})^{T}\bar{\boldsymbol{S}}^{\alpha}_{e},\boldsymbol{\kappa}^{e}_{i}\Big>\bar{j}_{e0}d\mathcal{A}_{e},
\end{align}
where $d\A_e=d\xi_e^1d\xi_e^2$, $\bar{j}_{e0}$ is the surface Jacobian at the element $e$ and the matrix
\begin{align}\label{E72}
\bar{\mathbf{K}}_{e\alpha}^{i} = \Big[\dfrac{\partial N^{i}}{\partial \xi^{\alpha}_e}\mathbf{I}_{6\times6}+N^{i}\ad_{\bar{\boldsymbol{\zeta}}_{t\alpha}^{e}}\Big].
\end{align}
Here, $\bar{\boldsymbol{\zeta}}_{t\alpha}^{e}$ is the deformation twist at the center of element $e$ and $\bar{\boldsymbol{S}}^\alpha_e\coloneqq \boldsymbol{\mathcal{D}}^{\alpha\beta}(\bar{\boldsymbol{\zeta}}_{t\alpha}^e-\boldsymbol{\zeta}_{0\alpha}^e)$. Integrals are evaluated on the element surface $\A_{e0}$.
Similarly, the virtual work of the external applied forces at the element level is given by:
\begin{align}\label{E73}
\left.\delta W_{ext}\right|_e = \int_{\mathcal{A}_{e0}}\sum_{i=1}^{4}\Big<N^{i}\bar{\mathbf{F}}^{ext},\boldsymbol{\kappa}_{i}^{e}\Big>\bar{j}_{e0}d\mathcal{A}_{e0} + \int_{\partial \mathcal{A}_{e0}}\sum_{i=1}^{4}\Big<N^{i}\mathbf{F}^{ext},\boldsymbol{\kappa}_{i}^{e}\Big>\bar{j}_{e0}ds_{e}.
\end{align}
By collecting terms, the total nodal unbalanced force vector for each element becomes:
\begin{align}\label{E74}
FU^{i}_{e} \coloneqq -\int_{\mathcal{A}_{e0}}N^{i}\bar{\mathbf{F}}^{ext}\bar{j}_{e0}d\mathcal{A}_{e} - \int_{\partial \mathcal{A}_{e0}}N^{i}\mathbf{F}^{ext}\bar{j}_{e0}ds_{e}-\int_{\mathcal{A}_{e0}}\sum_{\alpha=1}^{2}(\bar{\mathbf{K}}_{e\alpha}^{i})^{T}\bar{\boldsymbol{S}}^{\alpha}_{e}\bar{j}_{e0}d\mathcal{A}_{e}.
\end{align}
Substituting Eqs. (\ref{E70-3}) and (\ref{E72}) into the linearized internal virtual work leads to expressions for the material and geometric tangent stiffness matrices, respectively:
\begin{align}\label{E75}
K_{eM}^{ij} &\coloneqq \int_{\mathcal{A}_{e0}}\sum_{\alpha=1}^{2}\sum_{\beta=1}^{2}(\bar{\mathbf{K}}_{e\alpha}^{i})^{T}\boldsymbol{\mathcal{D}}^{\alpha \beta}\bar{\mathbf{K}}_{e\beta}^{j}\bar{j}_{e0}d\mathcal{A}_{e},\\ K_{eG}^{ij} &\coloneqq -\int_{\mathcal{A}_{e0}}\sum_{\alpha=1}^{2}N^{i}(\tilde{\ad}_{\bar{\boldsymbol{S}}_{\alpha}})^{T}\bar{\mathbf{K}}_{e\alpha}^{j}\bar{j}_{e0}d\mathcal{A}_{e}.
\end{align}
The contributions of the magnetic field to the unbalanced force vector and stiffness matrix are respectively obtained by substituting Eq. (\ref{E70-3}) into the linearized magnetic virtual work (Eq. \eqref{E80}):
\begin{align}\label{E81}
KM_{e}^{ij} \coloneqq \int_{\mathcal{A}_{e0}}\dfrac{1}{\mu_{0}}N^{i}(\mathbf{R}_0^T\mathbf{B}^{r}_{0})^\wedge(\bar{\mathbf{R}}_t^T\mathbf{B}^{a})^\wedge N^{j}\bar{j}_{e0}d\mathcal{A}_{e},
\end{align}
\begin{align}\label{E82}
FM_{e}^{i} \coloneqq \int_{\mathcal{A}_{e0}}\dfrac{1}{\mu_{0}}(\mathbf{R}_0^T\mathbf{B}^{r}_{0})\times (\bar{\mathbf{R}}^T_t \mathbf{B}^{a})N^{i}\bar{j}_{e0}d\mathcal{A}_{e}.
\end{align}
Neglecting inertia terms and assembling the expressions from Eqs. (\ref{E74}), (\ref{E75}), (\ref{E81}), and (\ref{E82}) into the discrete weak form (Eq. \ref{E61}), the element-level balance equation reads:
\begin{align}\label{E83}
\sum_{j=1}^{4}(K_{eM}^{ij}+K_{eG}^{ij}-KM_{e}^{ij})\boldsymbol{\eta}_{j}^{e}=FU^{i}_{e}-FM_{e}^{i}. 
\end{align}
Finally, by assembling the contributions from all finite elements, the global system of equilibrium equations is obtained. Introducing the vector of nodal incremental displacements and rotations, $\boldsymbol{\eta}T \in \mathbb{R}^{6N_n}$, where $N_n$ denotes the total number of nodes, the resulting system takes the form:
\begin{align}\label{E84}
(K_{MG}-KM)\boldsymbol{\eta}_{T}=FU-FM, 
\end{align}
where $K_{MG}$ represents the global stiffness matrix, assembled from both material and geometric contributions, and is symmetric at equilibrium. The vector $FU$ denotes the assembled unbalanced wrenches resulting from internal stresses and applied loads. The matrix $KM$ corresponds to the assembled magnetic stiffness matrix, and $FM$ represents the generalized unbalanced magnetic couple forces arising from magneto-mechanical interactions.
\begin{remark}
    To ensure that the numerical results at each iteration remain close to the correct solution, the norm of the residual of the equilibrium equation (\ref{E84}) should approach zero. This criterion must be verified at every iteration.
\end{remark}


To advance the iterative solution procedure, the shell configuration and associated deformation twist field are updated at each iteration step. The corresponding finite element algorithm, incorporating spatial integration over the shell domain, is outlined below.

\begin{algorithm}
\caption{Spatial Integration Scheme for Geometrically Exact Cosserat Shell FEM}
\begin{algorithmic}[1]
\Require Initial configuration $\mathbf{g}_0$, deformation twists $\boldsymbol{\zeta}_{0\alpha}$
\Ensure Updated configuration

\For{each iteration $I$}
    \Repeat
        \State Compute the stress resultants for use in the stiffness matrix calculation:
        \[
        \boldsymbol{S}^{\alpha} = \boldsymbol{\mathcal{D}}^{\alpha \beta} (\boldsymbol{\zeta}_{\beta}^i - \boldsymbol{\zeta}_{0\beta})
        \]
        \State Compute and assemble stiffness, magnetic matrices, and out-of-balance forces ($K_{MG}, KM, FU, FM$)
        \State Perform spatial integration at Gauss points and assemble global matrices
        \State Solve Eq.~(\ref{E84}) for $\boldsymbol{\eta}_{T}=(\boldsymbol{\eta}_1^T,\boldsymbol{\eta}_2^T,\dots,\boldsymbol{\eta}_e^T)^T$
        \State Update configuration:
        \[
        \mathbf{g}_{te} \gets \mathbf{g}_{te} \cdot \exp(\hat{\boldsymbol{\eta}}_e)
        \]
        \State Update deformation twist:
        \[
        \hat{\boldsymbol{\zeta}}_{\alpha e} \gets 
        \exp(-\hat{\boldsymbol{\eta}}_e) \hat{\boldsymbol{\zeta}}_{\alpha e} \exp(\hat{\boldsymbol{\eta}}_e) 
        + \exp(-\hat{\boldsymbol{\eta}}_e) \frac{\partial \exp(\hat{\boldsymbol{\eta}}_e)}{\partial \xi^\alpha}
        \]
        \State Compute convergence norm:
        \[
        \|FU-FM\| \gets \left( \sum_{k=1}^{N_n} \|(FU-FM)_k\|^2 \right)^{1/2}
        \]
    \Until{$\|FU-FM\| \leq \text{Tolerance}$}
    \State Update iteration: $I \gets I + 1$
\EndFor

\end{algorithmic}
\end{algorithm}


\section{Benchmark Numerical Studies and Experimental Validations}\label{sec:exp}
In this section, we explore a set of numerical simulations demonstrating the effectiveness of the previously explained formulation. In all simulations, a two-dimensional four-node element with $2\times 2$ Gauss integration rule has been employed. 

\begin{remark}
Shear-locking is a well-known numerical artifact in finite element analysis, particularly in the modelling of thin shell or beam structures. It arises when the discrete formulation over-constrains the transverse shear deformation, leading to an artificially stiff response and poor convergence, especially as the element thickness decreases. To mitigate the shear-locking phenomenon in the proposed formulation, I employ the material twist values evaluated at the element centroid nodes to compute the updated stiffness matrix, rather than using conventional reduced integration techniques. This approach allows for a more accurate representation of the deformation modes without sacrificing stability. Numerical results demonstrate that this technique achieves consistent convergence as the number of elements or nonlinear solver iterations increases, validating the effectiveness of the method in eliminating shear-locking effects.
\end{remark}

\begin{remark}
The strategy used to alleviate shear locking in this work is conceptually inspired by the constant curvature method, which assumes that strain measures, such as curvature or twist, remain constant over an element. In the proposed formulation, this idea is applied at the element level by evaluating the strain twist, particularly the material twist, at the centroid of each element and extending this value uniformly throughout the element for the purpose of constructing the tangent stiffness matrix. This approach reduces the artificial stiffness typically introduced by conventional finite element formulations, allowing for a more accurate representation of transverse shear deformation. As a result, the method provides a simple yet effective way to improve numerical performance and solution quality in problems involving thin shells.
\end{remark}

\subsection{A Cantilever Plate Subject to an End Shear Force}
The large deflection of a cantilever plate with a Young's modulus of $E = 200 GPa$, a material length of $1000 mm$, a thickness of $10 mm$, and a width of $200 mm$ \cite{beheshti2016} under an applied non-follower shear force is investigated. The numerical results are compared with available data in the literature in Figure \ref{fig33}. In this study, we consider $20$ four-node elements.
\begin{figure}[h]
\includegraphics[scale=0.4]{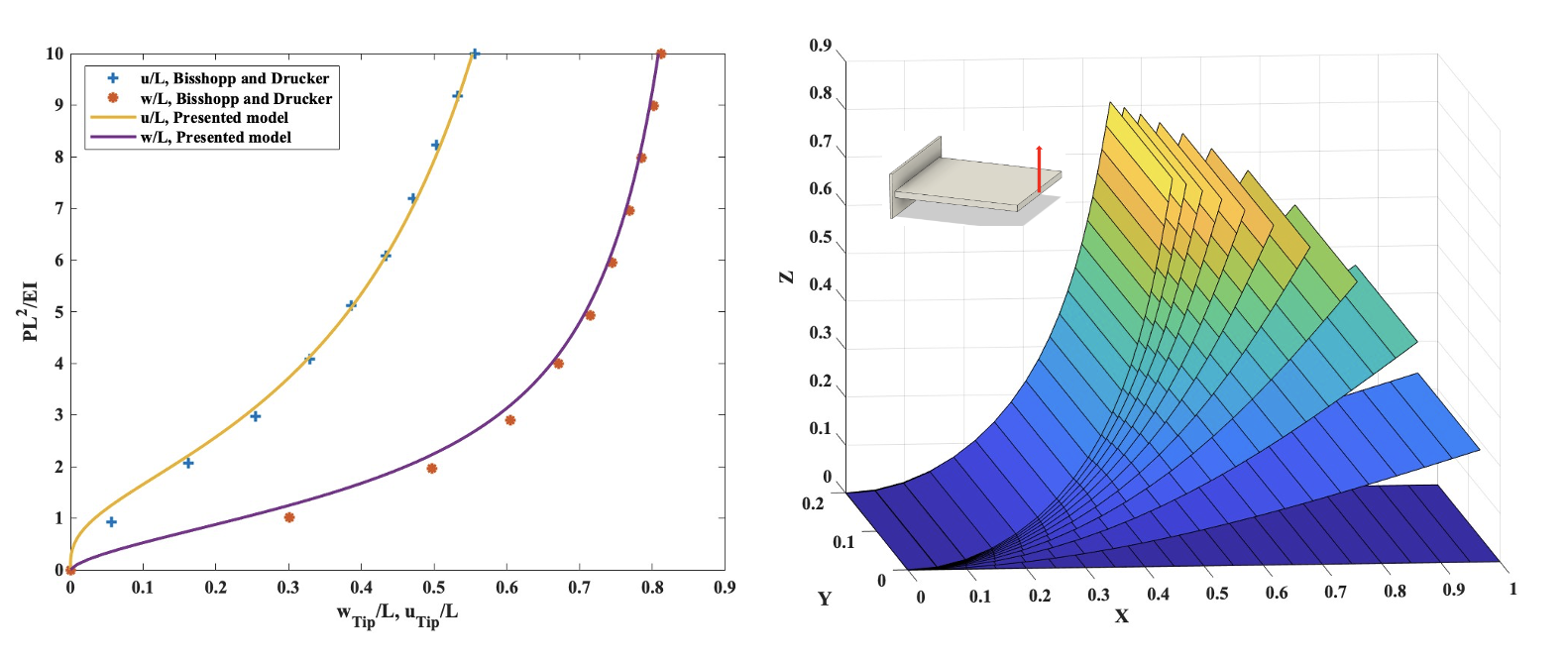}
\centering
\caption{Load-deflection of the cantilever plate}
\label{fig33}
\end{figure}
\subsection{Pure Bending of a Cantilever Plate}
In this problem, we examine a straight Cosserat shell characterized by material properties $E=12\times 10^{6}$. This shell is exposed to a concentrated end moment $M$. When $M=2\pi EI/L$ (with $L$ representing the length of the shell and $I$ denoting the moment of inertia), the shell undergoes a complete transformation into a circular shape. The numerical results for this scenario are depicted in Figure \ref{fig3}, considering $L=10$, width $w=1$, and thickness $h=0.1$.

\begin{figure}[h]
\includegraphics[scale=0.12]{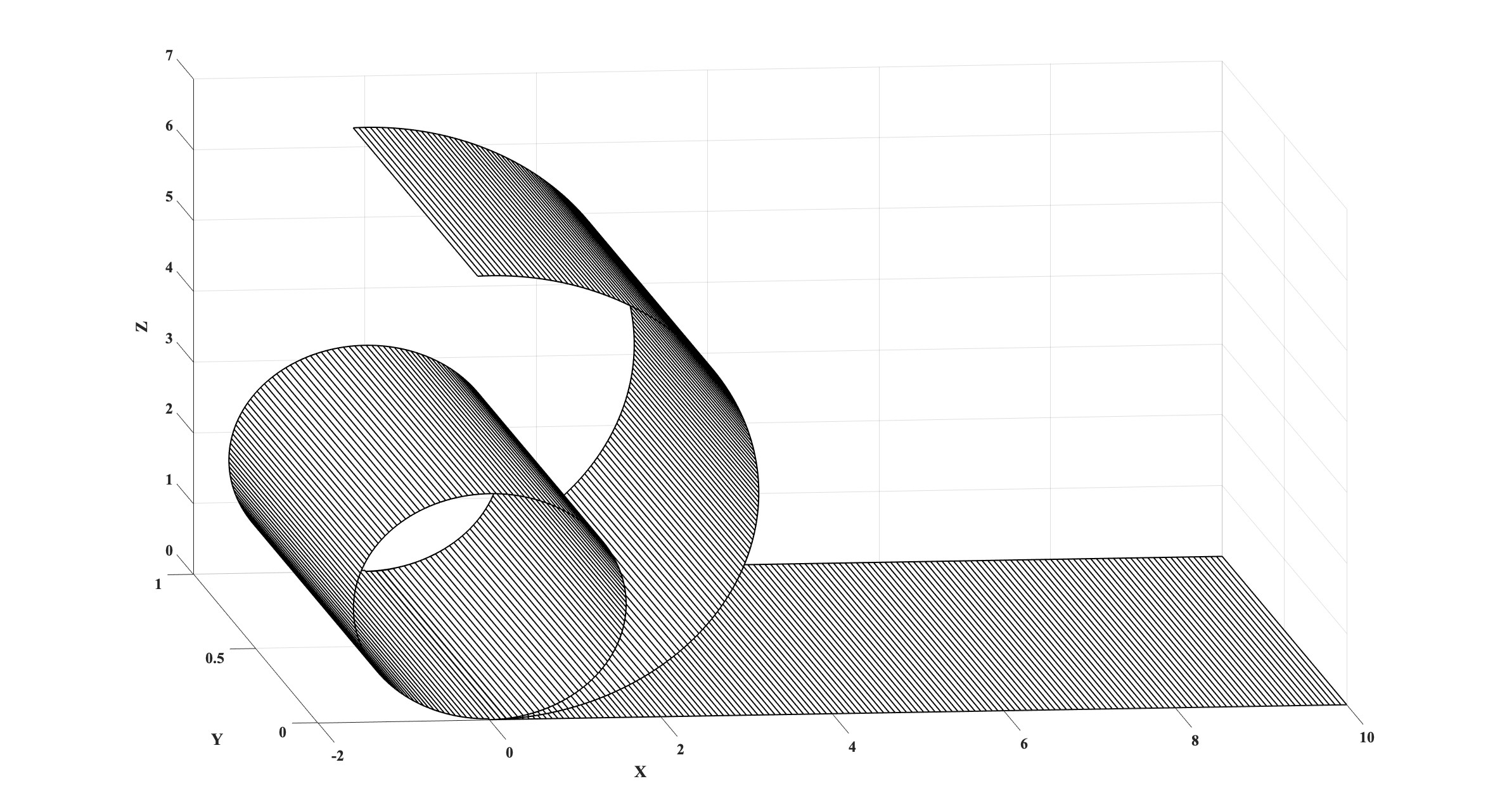}
\centering
\caption{Roll-up of the clamped plate for the end rotation $2\pi$ with 150 elements}
\label{fig3}
\end{figure}
To showcase the effectiveness of our proposed model, we investigated the response of a clamped plate subjected to moments equal to twice ($M=2\pi EI/L$) and treble of this value. Figures \ref{fig4} and \ref{fig44} display the resulting configurations of the shell and the tip deflection, respectively. The results in Figure \ref{fig44} are compared with available numerical results, indicating a good agreement. Additionally, we examined the behavior of the shell under drilling rotation by using a similar cantilever configuration subjected to in-plane bending. In this case, the applied load is in the form of a drilling moment. It is important to mention that we used the reduced integration technique solely for in-plane bending to alleviate in-plane shear locking. Numerical results for the Drilling moment are presented in Figure \ref{fig5} for $E=1200$, $\nu=0$, $h=1$, $w=1$, and $L=10$.
\begin{figure}[hbt!]
\includegraphics[scale=0.4]{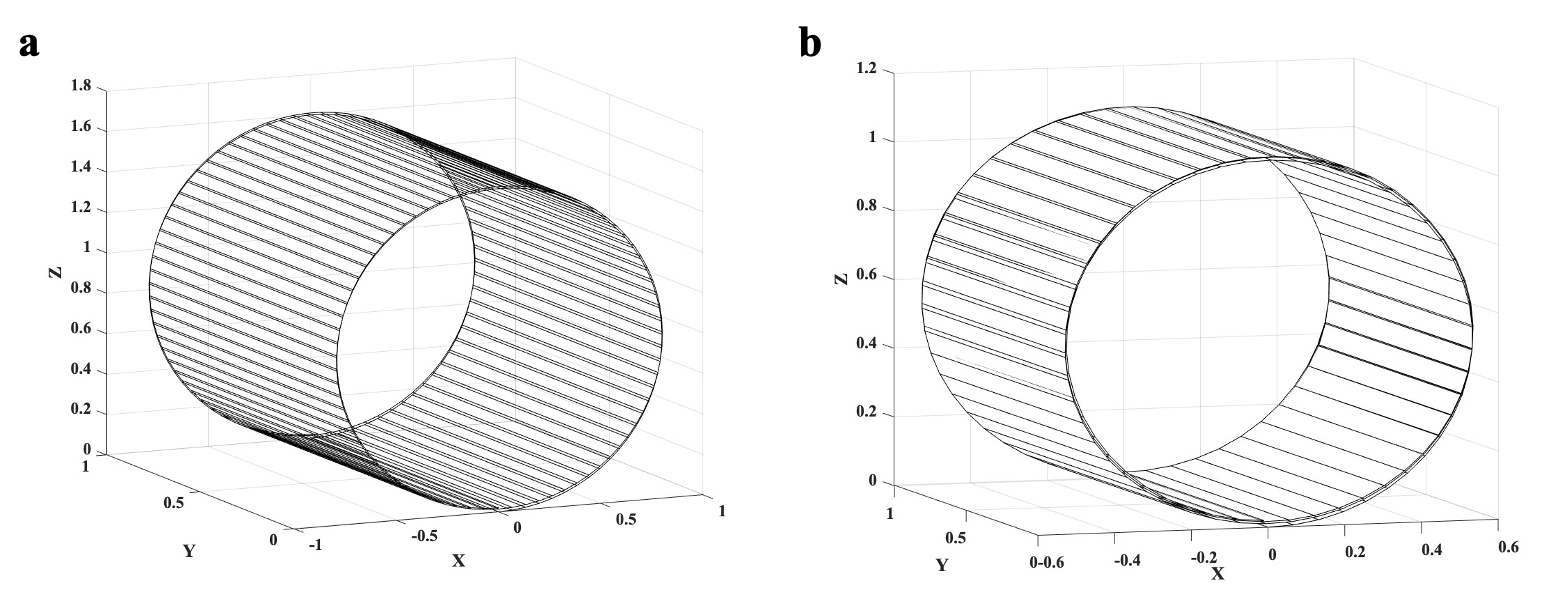}
\centering
\caption{Roll-up of the clamped plate for the end rotation: (a) $4\pi$ with 150 elements and (b) $6\pi$ with 150 elements}
\label{fig4}
\end{figure}

\begin{figure}[hbt!]
\includegraphics[scale=0.3]{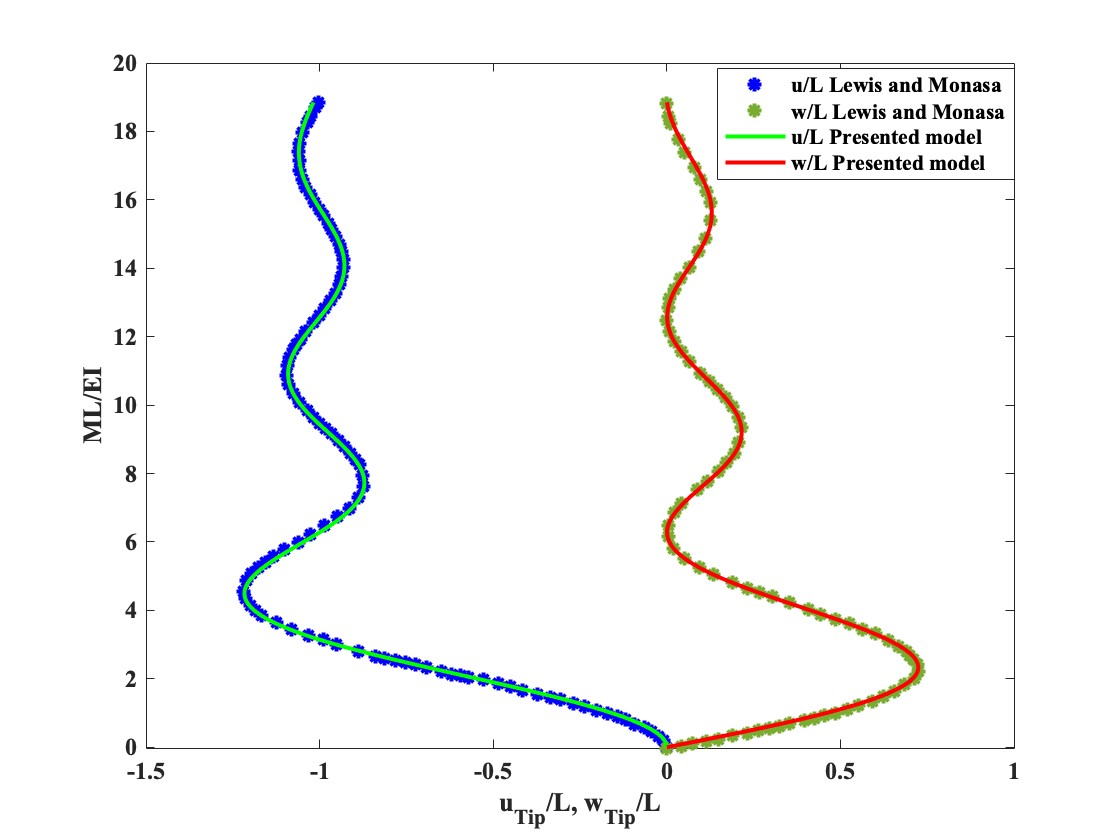}
\centering
\caption{Load-deflection curves of the clamped plate subject to an end
moment for a rotation of $6\pi$ with 150 elements}
\label{fig44}
\end{figure}

\begin{figure}[hbt!]
\includegraphics[scale=0.45]{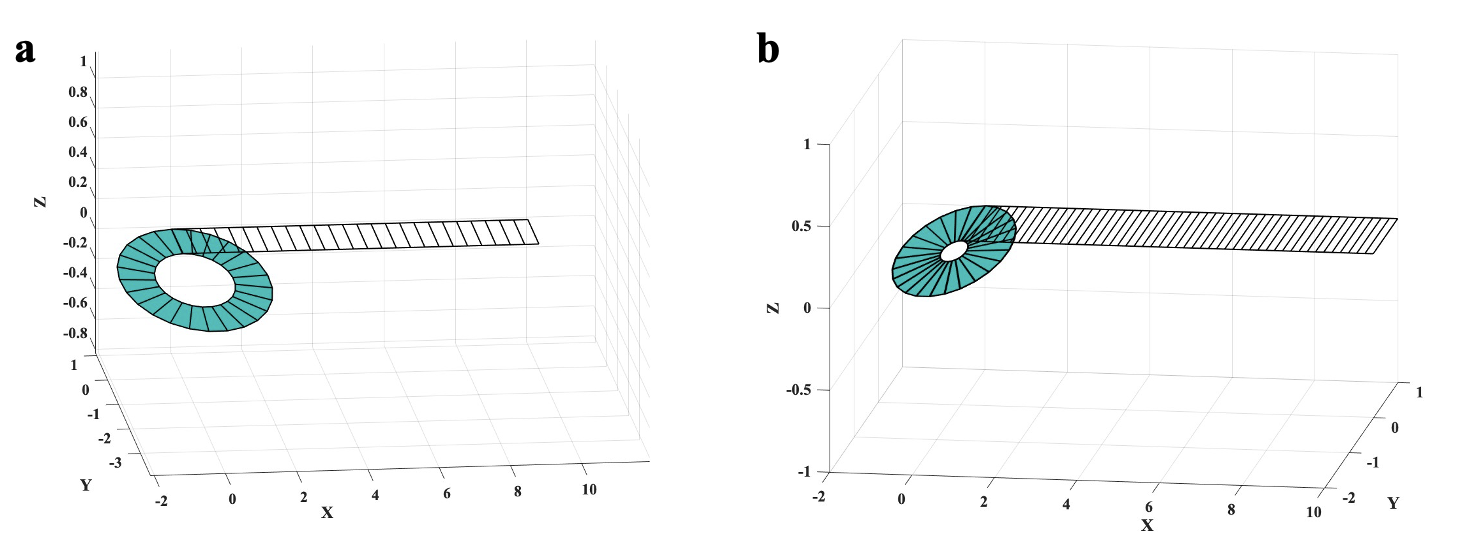}
\centering
\caption{In-plane bending of the clamped plate for the end rotation: (a) $2\pi$ with 25 elements and (b) $4\pi$ with 50 elements}
\label{fig5}
\end{figure}

\subsection{Torsion of a Plate}
To assess the effectiveness of the proposed approach in handling significant rotations, displacements, and ensuring numerical stability, we examined the behavior of the Cosserat shell under torsional stress. In this study, a torsional moment was applied to the end of a flat plate, resulting in relative torsional rotations of approximately $180^{\circ}$, $360^{\circ}$, and $540^{\circ}$. The material properties considered were $E=12\times 10^6$ and $\nu = 0.3$, with $L=1$, width is $w=0.25$ and thickness $h=0.1$. The accompanying Figure \ref{fig6} depicts the deformed configurations.
\begin{figure}[hbt!]
\includegraphics[scale=0.45]{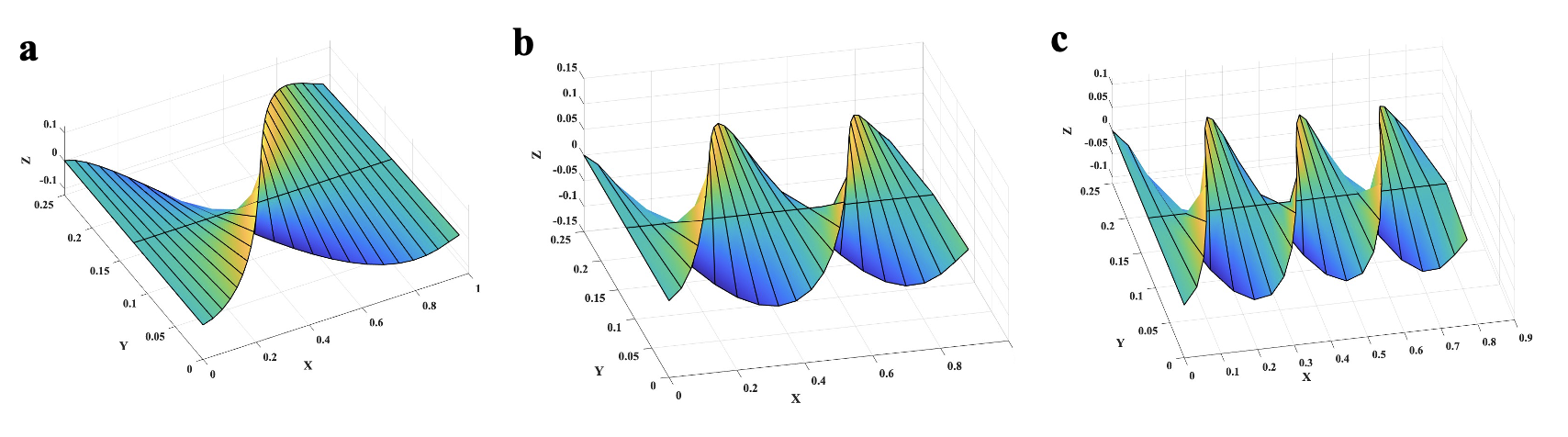}
\centering
\caption{The torsion of the clamped plate for the end rotation: (a) $\pi$ with 50 elements, (b) $2\pi$ with 50 elements, and (c) $3\pi$ with 50 elements}
\label{fig6}
\end{figure}


To demonstrate the capability of the proposed model in simulating large rotations, a combination of torsional and bending moments is applied to the boundary of the clamped plate. The results are shown in Figure~\ref{figmix-3}, using material parameters $E = 12 \times 10^6$ and $\nu = 0.3$, with geometric dimensions $L = 10$, width $w = 1$, and thickness $h = 0.1$. In this simulation, a mesh of 50 elements in the $X$ direction and 20 elements in the $Y$ direction is used.
\begin{figure}[H]
\includegraphics[scale=0.25]{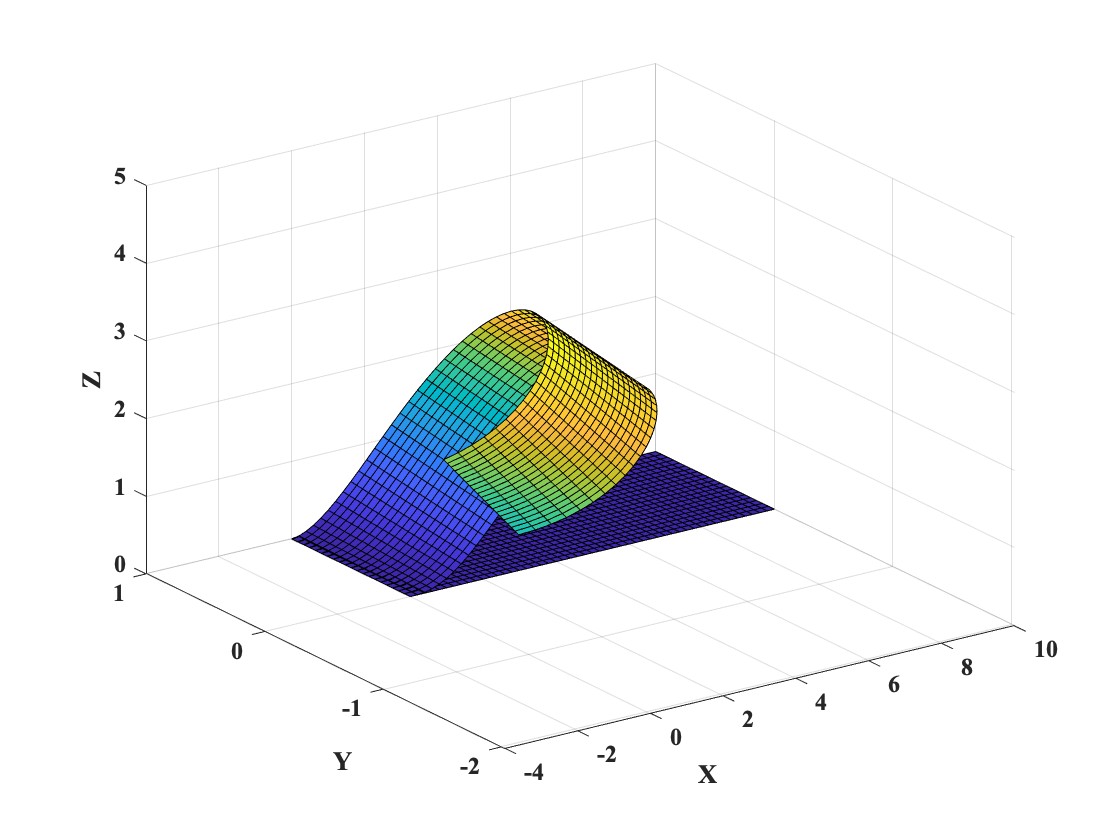}
\centering
\caption{Mixed torsion and roll-up of the clamped plate}
\label{figmix-3}
\end{figure}
\subsection{Large Deformation of a Cantilever Circular Arch Shell}
In this example, the large deformation of a shell-like circular arch is studied under a nonconservative applied load. The shell is made of isotropic material with $E = 7.2 \times 10^{10} N/m^{2}$ , cross-sectional area $A = 10^{-4}m^{2}$ , and moment of inertia $5 \times 10^{-9}m^{4}$  \cite{dadgar2021, argyris1981}. The initial configuration of the shell is a half-cylinder with a radius of $R=0.5m$. The elastic behavior of the shell is presented under two different load conditions. Figure \ref{fig66} indicates the load-deflection results of the shell under tangent following load on the shell edge, and Figure \ref{fig666} indicates the results for a normal load on the tip of the shell.

\begin{figure}[hbt!]
\includegraphics[scale=0.45]{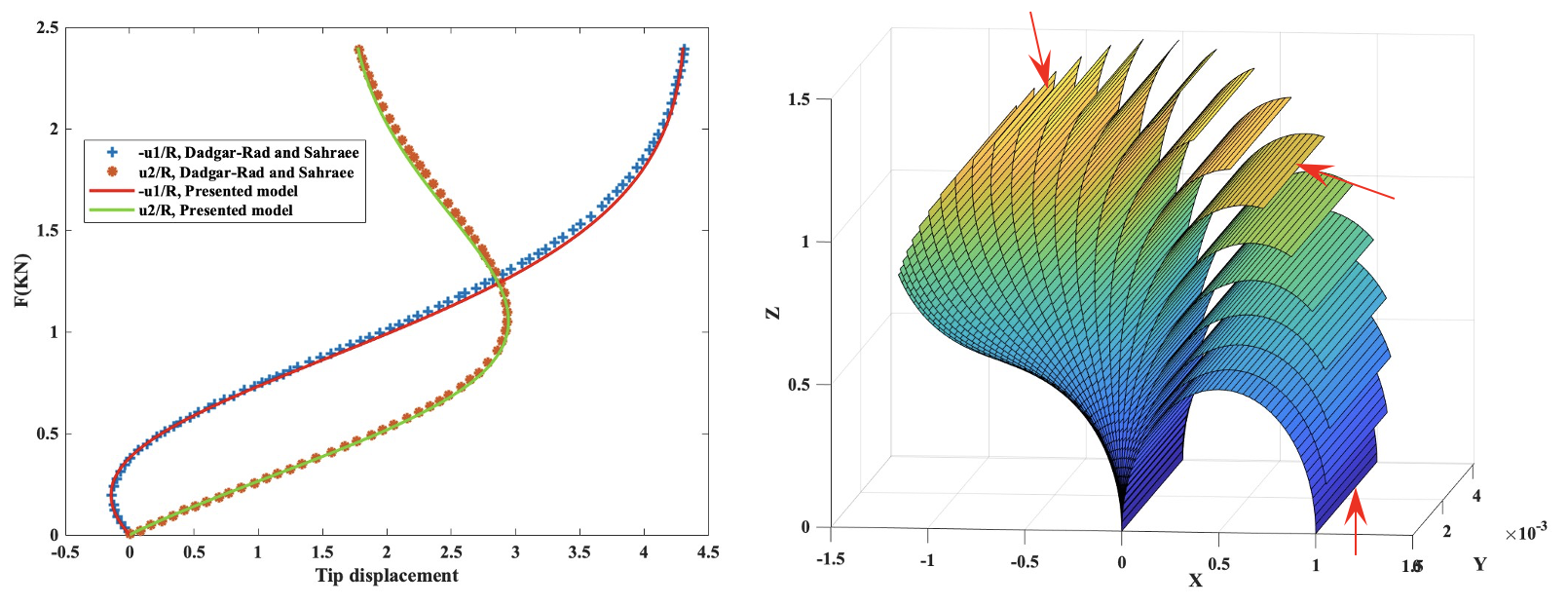}
\centering
\caption{Shell-like circular arch under follower tangential end force with 50 elements}
\label{fig66}
\end{figure}

\begin{figure}[hbt!]
\includegraphics[scale=0.45]{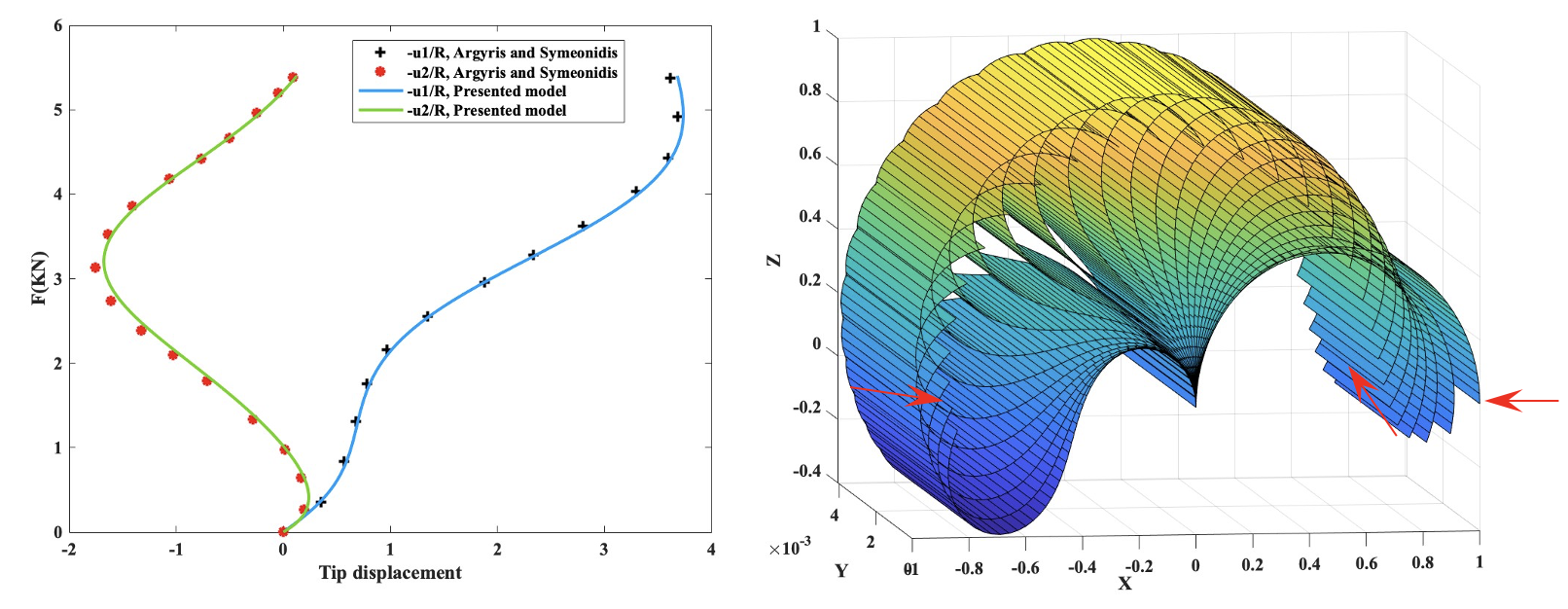}
\centering
\caption{Shell-like circular arch under follower transverse end force with 50 elements}
\label{fig666}
\end{figure}

\subsection{Pure Bending of a Cantilever Shell-Like Circular Arch}
In this problem we examine the large deformation of curved shell under applied external pure moment. The shell reference configuration is half-cylender shell which is clamped in one side. The material properties are similar one presented in section (6.2) with a curved with radius $R=10/2\pi$, width $w=1$, and thickness $h=0.1$. The numerical results are presented in Figure \ref{figg33}
\begin{figure}[h]
\includegraphics[scale=0.25]{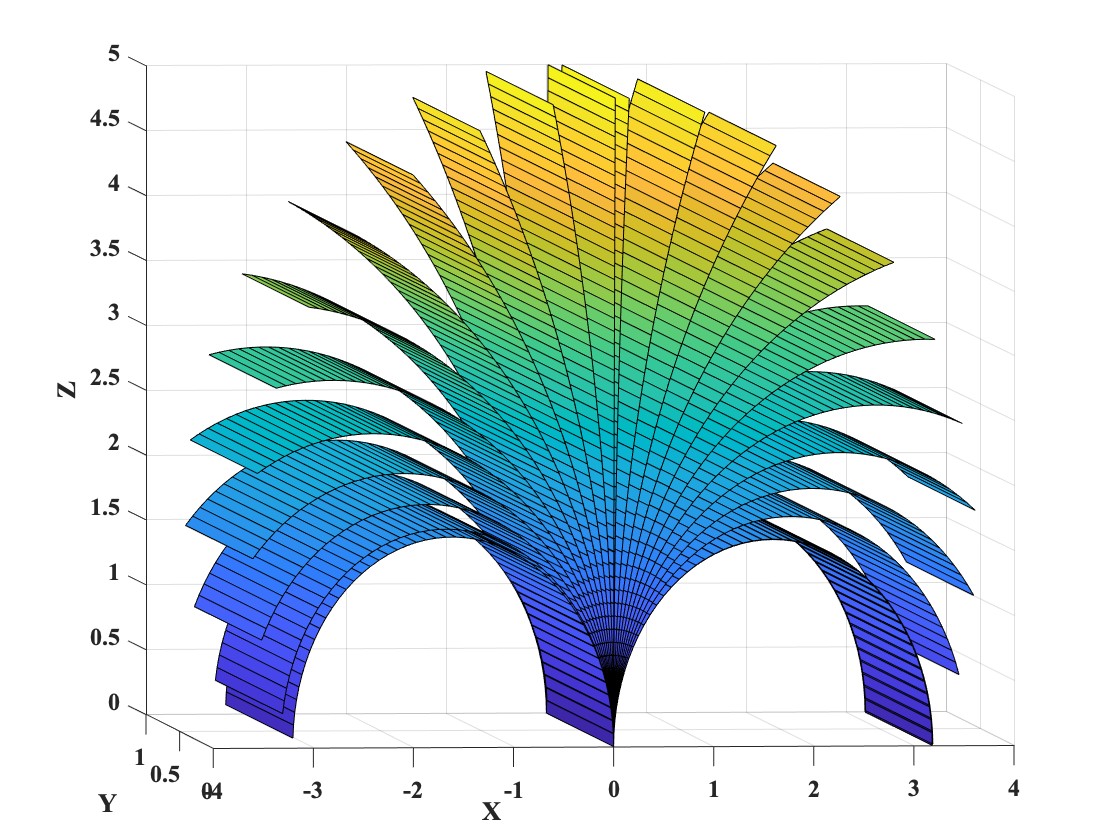}
\centering
\caption{Roll-up of a cantilever Shell-like circular arch for the end rotation $2\pi$ with 50 elements}
\label{figg33}
\end{figure}

\subsection{Magnetized Shells}
In this section, the response of the cantilever magnetized plates under applied uniform external magnetic field is investigated.
To confirm the accuracy of our formulated shell model for hard-magnetic elastica, we verify our predictions by comparing them with both finite element simulations and experimental findings previously documented by \cite{zhao2019, dadgar2022fi}. We also conducted an experimental study involving a thin elastic plate containing embedded hard-magnetic (NdFeB) particles.

\subsubsection{Validation with Experimental Data} 
To conduct experimental testing, we produced a hard magnetic elastic plate by blending Mold Star 31T and neodymium-iron-boron (NdFeB) permanent magnet material in a 1:1 and 2:1 mass ratio. Mold Star 31T, with a Young's modulus of $E=324.054 kPa$, was combined with NdFeB particles, each approximately $30 \mu m$ in size. The fabrication process involved pouring the mixture into a square plate-shaped plastic mold. Following the methodology outlined in \cite{diller2014}, the magnetic plate was exposed to a uniform magnetic field of approximately $1 T$ , oriented perpendicular to the plate, after the curing period. Subsequently, upon removal from the magnetic field, the hard magnetic plate was magnetized in a direction perpendicular to the plate.
Using this procedure, we produced two samples, designated as Sample A and Sample B. Both samples shared dimensions of thickness $0.7 mm$, length $40 mm$, and width $30 mm$.
The elastic properties of the hard magnetic plate were determined using the formula proposed by \cite{kim12019}.
\begin{align}\label{E85}
E=E_{0}exp\Big(\dfrac{2.5\phi}{1-1.35\phi}\Big), 
\end{align}
where $\phi$ denotes the particle volume fraction, and $E_{0}$ is the Young modulus of pure Moldstar without particles.
Sample A, with a magnetic particle volume fraction of $6\%$, exhibited a modulus of elasticity of $1.19E_{0}$ and a measured magnetized magnitude of $B^{r} =2.5 mT$. In contrast, Sample B, with a magnetic particle volume fraction of $12\%$, displayed a modulus of elasticity of $1.45E_{0}$ and a measured magnetized magnitude of $B^{r} =6.8 mT$.

 To assess the presented Cosserat shell model's capability in describing the large deformation of hard-magnetic plates, we compare the numerical and experimental results of the magnetic plate's deformation under an external magnetic field with a magnitude of $B^{a}=9.5 mT$ perpendicular to the internal magnetic field. The numerical findings depicted in the Figure \ref{fig10} align well with the experimentally derived data.
\begin{figure}[hbt!]
\includegraphics[scale=0.4]{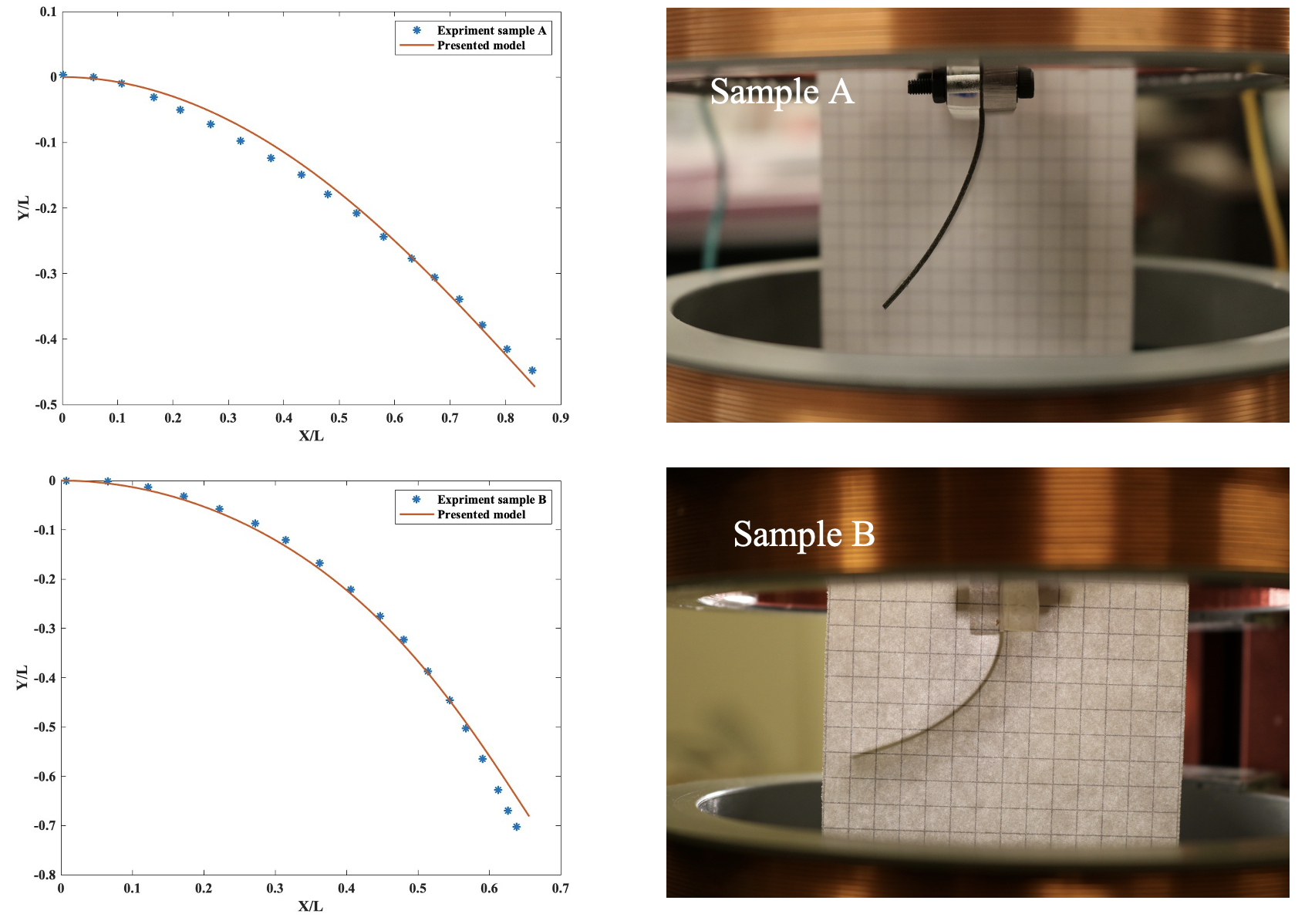}
\centering
\caption{Large deformation of a suspended square plate under a uniform magnetic field.}
\label{fig10}
\end{figure}

\subsubsection{Validation with Available Data}
Based on existing experimental data, we examine four cantilever ferromagnetic elastic beams with lengths $L={11,19.2,17.2,17.2}$ (mm), thicknesses $h={1.1,1.1,0.84,0.42}$ (mm), and a width of $w=5$ mm \cite{zhao2019, dadgar2022fi}. Following the approach in \cite{zhao2019, dadgar2022fi}, we adopt a shear modulus of $\mu = 303$ kPa, a Lame constant of $\lambda = 7300$ kPa, and an applied external uniform magnetic flux of $B^{a}=0.05$ T perpendicular to the remnant magnetic flux $B^{r}=0.143$ T. The maximum tip deflection of the beam is computed for various ratios of $L/h$, and the obtained results are compared with existing literature data in the accompanying figure \ref{fig8}. In this illustration, we utilize 30, 15, 15, and 10 elements, correspondingly, for different $L/h$ ratios, where $L/h = {41, 20.5, 17.5, 10}$. The numerical results closely align with the experimental data for large $L/h$.  

The Cosserat shell model presented here is capable of capturing the deformation of a hard magnetic plate under an antiparallel magnetic field. The numerical results are juxtaposed with experimental data in Figure \ref{fig9}.
\begin{figure}[hbt!]
\includegraphics[scale=0.45]{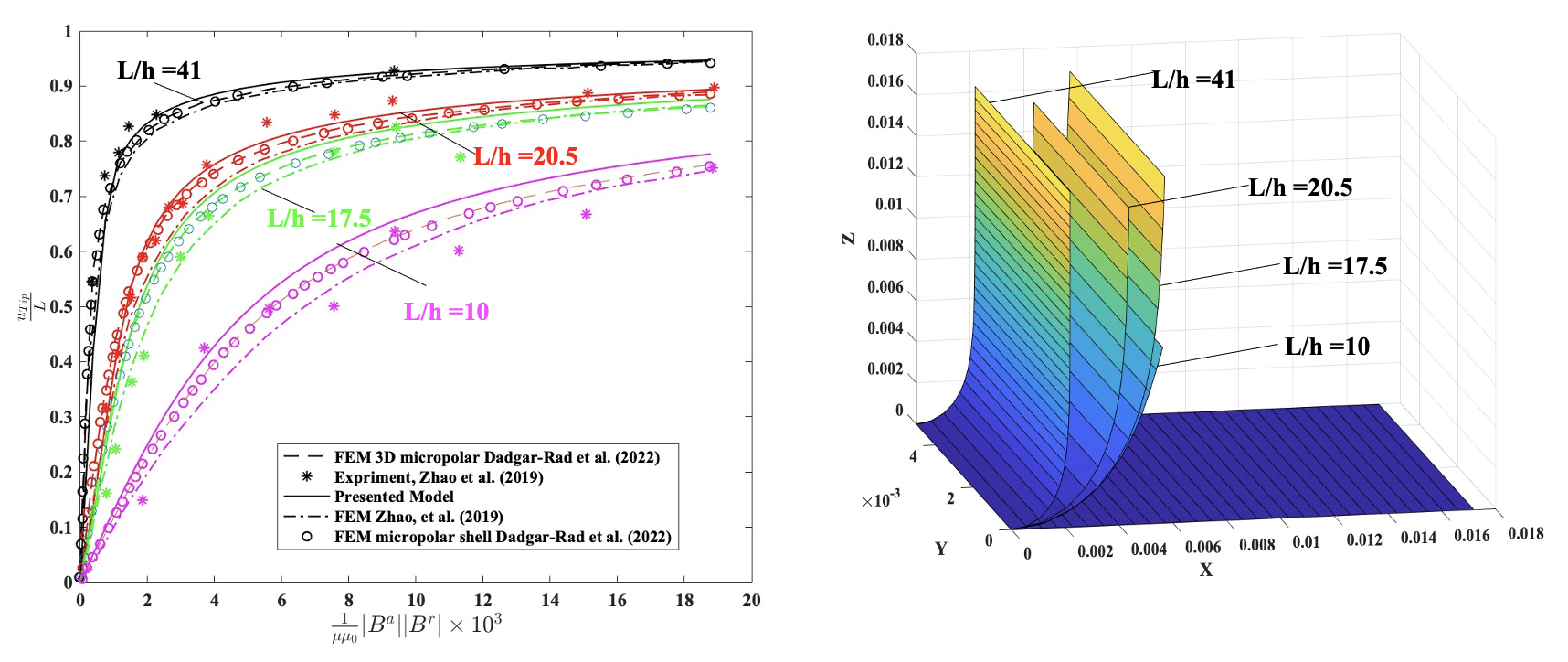}
\centering
\caption{Cosserat plate under magnetic fields}
\label{fig8}
\end{figure}

\begin{figure}[hbt!]
\includegraphics[scale=0.45]{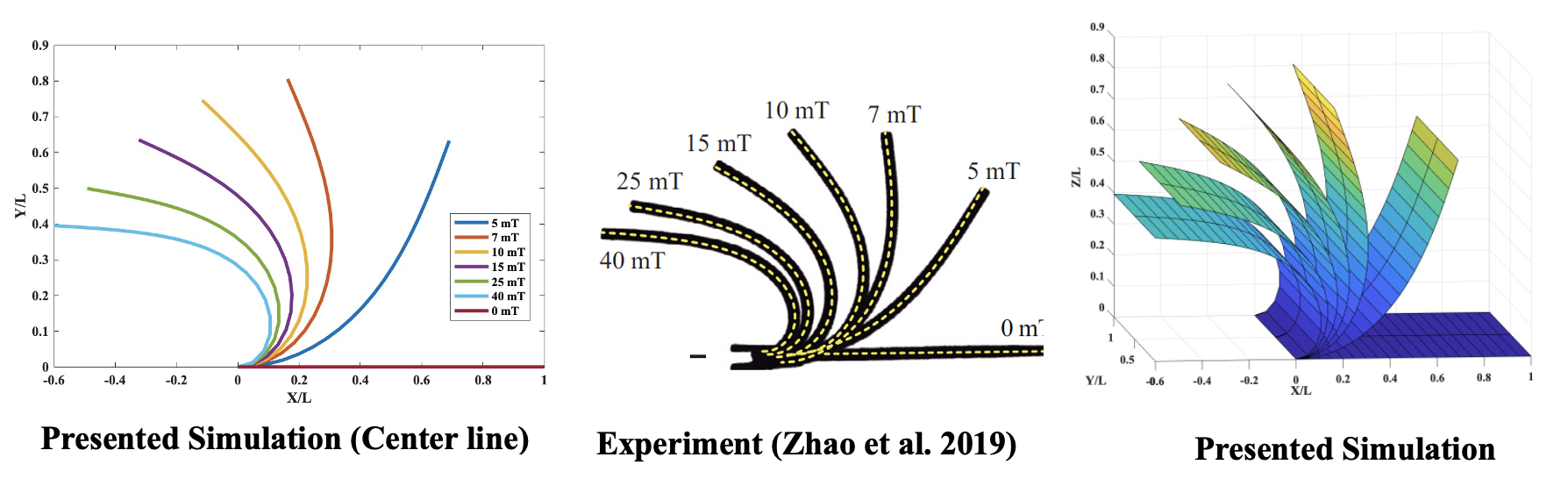}
\centering
\caption{Cosserat plate under antiparallel magnetic fields}
\label{fig9}
\end{figure}


\subsection{Soft Magnetic Gripper}  
Based on the presented fabrication method, we designed and tested a simple soft magnetic gripper, as illustrated in Figure~\ref{fig11}. The gripper consists of flexible plates embedded with magnetic particles, allowing it to deform and grasp objects when exposed to an external magnetic field. The proposed numerical model can be directly applied to simulate the deformation of the gripper under various loading and actuation conditions.  

To evaluate its performance, we analyzed the gripper’s ability to conform to objects of different shapes and sizes. The numerical simulations captured the nonlinear deformation patterns and provided predictions for tip displacement, which showed good agreement with experimental observations. These results demonstrate the capability of the model to accurately predict the gripper’s behavior, thereby validating its potential for use in the design and optimization of soft robotic manipulators.

\begin{figure}[hbt!]
\includegraphics[scale=0.5]{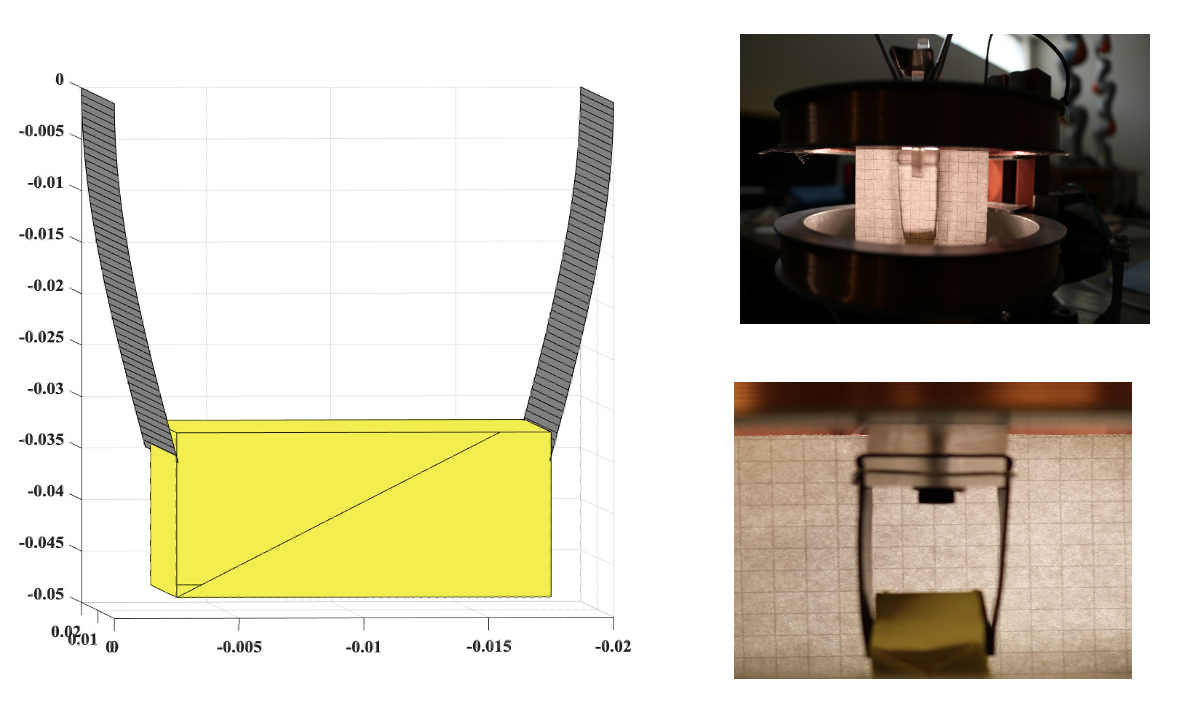}
\centering
\caption{Soft magnetic gripper}
\label{fig11}
\end{figure}

\section{Conclusions}\label{sec:conc}
A novel geometrically exact finite element framework for Cosserat shells with six degrees of freedom was developed using Lie group theory to robustly capture large deformations and rotations. The static balance equations were derived in both strong and weak forms via the principle of virtual work. Based on the weak form, a singularity-free Lagrangian finite element formulation was proposed, effectively eliminating shear-locking without additional computational cost. The model was extended to magneto-active shell-like soft robots by incorporating magnetic actuation forces. The resulting static equilibrium equations were solved numerically, and the results showed strong agreement with analytical predictions and conventional finite element models. Additionally, a plate-shaped hard magnetic elastomer was fabricated, and experimental validation confirmed the accuracy and applicability of the proposed model.







\textbf{Data availability:} The datasets generated and
analyzed during the current study are available upon request from the corresponding author.
\\

\textbf{Conflict of interest:}  The authors declare that they have no conflict of interest.
\\

\textbf{Author contributions:} M.J. conducted the research as part of his doctoral studies under the supervision of R.C. M.J. developed the theoretical framework, performed the mathematical analysis, designed the algorithms, and carried out the simulations and experiments. R.C. provided the research vision, secured funding, guided the overall direction of the work, and revised the manuscript. M.J. prepared the initial draft of the manuscript, and both authors reviewed and approved the final version.

\textbf{Funding:}  This work was partially supported by
the Natural Sciences and Engineering Research Council of Canada (grant no: DGECR-2019-00085).

\appendix
\section{Proof of Lemma 1}\label{app:lem1}
To show the equivalence of the terms appearing in Eq. \eqref{E18} and \eqref{E19}, we need to compute the variation $\delta \boldsymbol{\zeta}_{t\alpha}$ induced on $\se$ by the variation $\delta \mathbf{g}_t$. This variation can be found from the chain rule, using the definition in Eq. \eqref{eq:zeta}:
\begin{equation}
    \delta \boldsymbol{\hat\zeta}_{t\alpha}= \delta \big(\mathbf{g}_t^{-1}\frac{\partial \mathbf{g}_t}{\partial\xi^\alpha}\big)
\end{equation}
Knowing
that for a curve on $\SE$ the order of the partial derivatives is interchangeable, \begin{align}\label{E24}
{\delta\boldsymbol{\hat\zeta}}_{t\alpha} =-\mathbf{g}_{t}^{-1}\delta \mathbf{g}_{t}\mathbf{g}_{t}^{-1} \dfrac{\partial \mathbf{g}_{t}}{\partial \xi^{\alpha}}+\mathbf{g}_{t}^{-1}\dfrac{\partial \delta \mathbf{g}_{t}}{\partial \xi^{\alpha}}.
\end{align}
On the other hand, taking the derivative of Eq. (\ref{E21}) with respect to $\xi^{\alpha}$ leads to
\begin{align}\label{E26}
\mathbf{g}_{t}^{-1}\dfrac{\partial \delta \mathbf{g}_{t}}{\partial \xi^{\alpha}}=\dfrac{\partial \hat{\boldsymbol{\kappa}}}{\partial \xi^{\alpha}}+\mathbf{g}_{t}^{-1}\dfrac{\partial \mathbf{g}_{t}}{\partial \xi^{\alpha}}\hat{\boldsymbol{\kappa}},
\end{align}
Substituting Eq. (\ref{E26}) into (\ref{E24}), leads to
\begin{align}\label{E28}
\delta \hat{\boldsymbol{\zeta}}_{t\alpha} =\dfrac{\partial \hat{\boldsymbol{\kappa}}}{\partial \xi^{\alpha}}-\hat{\boldsymbol{\kappa}}\hat{\boldsymbol{\zeta}}_{t\alpha}+\hat{\boldsymbol{\zeta}}_{t\alpha} \hat{\boldsymbol{\kappa}}=\dfrac{\partial \hat{\boldsymbol{\kappa}}}{\partial \xi^{\alpha}}+[\hat{\boldsymbol{\zeta}}_{t\alpha}, \hat{\boldsymbol{\kappa}}]=\left(\dfrac{\partial {\boldsymbol{\kappa}}}{\partial \xi^{\alpha}}+\ad_{\boldsymbol{\zeta}_{t\alpha}}({\boldsymbol{\kappa}})\right)^\wedge,
\end{align}

\section{Proof of Proposition 1}\label{app:pro1}
To prove Eq. \eqref{E67}, we substitute the perturbed configuration $\mathbf{g}_{\varepsilon}$ into the functional $G_{int}(\mathbf{g}_{t},{\boldsymbol{\kappa}})$:
\begin{align}\label{EP1}
G_{int}(\mathbf{g}_{\varepsilon},{\boldsymbol{\kappa}})=\int_{\mathcal{A}_{0}}\Big(\sum_{\alpha=1}^2\Big<\boldsymbol{S}^{\alpha}_{\varepsilon},\dfrac{\partial \boldsymbol{\kappa}}{\partial \xi^{\alpha}}+\ad_{\boldsymbol{\zeta}_{t\alpha }^\varepsilon} \boldsymbol{\kappa}\Big>\Big)\bar{j}_{0}d\mathcal{A},
\end{align}
where $\boldsymbol{S}^{\alpha}_{\epsilon}=\sum_{\beta=1}^2\boldsymbol{\mathcal{D}}^{\alpha \beta}(\boldsymbol{\zeta}_{t\alpha}^\varepsilon-\boldsymbol{\zeta}_{0\alpha})$ and $\hat{\boldsymbol{\zeta}}_{t\alpha}^\varepsilon=\mathbf{g}_{\varepsilon}^{-1}\dfrac{\partial \mathbf{g}_{\varepsilon}}{\partial \xi^{\alpha}}$.
Taking the derivative of Eq. (\ref{EP1}) with respect to $\varepsilon$ and calculate the results at $\varepsilon =0$ leads to
\begin{align}
    \nonumber\left.\dfrac{\partial}{\partial \varepsilon}\right|_{\varepsilon=0}
\!\!\!\!&G_{int}(\mathbf{g}_{\varepsilon},\boldsymbol{\kappa})\\
=&\int_{\mathcal{A}_{0}}\sum_{\alpha=1}^2\Big(\Big<\left.\dfrac{\partial\boldsymbol{S}^{\alpha}_{\varepsilon}}{\partial \varepsilon}\right|_{\varepsilon=0}\!\!\!,\dfrac{\partial \boldsymbol{\kappa}}{\partial \xi^{\alpha}}+\ad_{\bar{\boldsymbol{\zeta}}_{t\alpha }} \boldsymbol{\kappa}\Big>+\Big<\bar{\boldsymbol{S}}^{\alpha},\dfrac{\partial \boldsymbol{\kappa}}{\partial \xi^{\alpha}}-\ad_{\boldsymbol{\kappa}}{\left.\dfrac{\partial\boldsymbol{\zeta}_{t\alpha }^\varepsilon}{\partial \varepsilon}\right|_{\varepsilon=0}} \!\!\!\Big>\Big)\bar{j}_{0}d\mathcal{A}
\end{align}
Knowing the definition of the strain $\boldsymbol{S}^\alpha$ from Eq. \eqref{eq:strain} and 
\begin{align}
    \left.\dfrac{\partial\boldsymbol{\zeta}_{t\alpha}^\varepsilon}{\partial\varepsilon}\right|_{\varepsilon=0}=\left.\dfrac{\partial}{\partial\varepsilon}\right|_{\varepsilon=0}\big(\mathbf{g}_{\varepsilon}^{-1}\dfrac{\partial \mathbf{g}_{\varepsilon}}{\partial \xi^{\alpha}}\big)^\vee=\big(-\hat{\boldsymbol{\eta}}\bar\g_t^{-1}\dfrac{\partial \bar\g_{t}}{\partial \xi^{\alpha}}+\bar\g_t^{-1}\dfrac{\partial \bar\g_{t}}{\partial \xi^{\alpha}}\hat{\boldsymbol{\eta}}+\dfrac{\partial\hat{\boldsymbol{\eta}}}{\partial\xi^\alpha}\big)^\vee=\bar{\mathbf{K}}_\alpha\boldsymbol{\eta},
\end{align}
we have

\begin{align}\label{EP2}
\left.\dfrac{\partial}{\partial \varepsilon}\right|_{\varepsilon=0}
\!\!\!\!&G_{int}=\int_{\mathcal{A}_{0}}\sum_{\alpha=1}^2\Big(\Big<\sum_{\beta=1}^2\boldsymbol{\mathcal{D}}^{\alpha \beta}\bar{\mathbf{K}}_{\beta}\boldsymbol{\eta},\bar{\mathbf{K}}_{\alpha}\boldsymbol{\kappa}\Big>+\Big<\bar{\boldsymbol{S}}^{\alpha},\ad_{\bar{\mathbf{K}}_{\alpha}\boldsymbol{\eta}}\boldsymbol{\kappa}\Big>\Big)\bar{j}_{0}d\mathcal{A}.
\end{align}
For the last term in Eq. (\ref{EP2}) we have
\begin{align}\label{EP22}
\Big<\bar{\boldsymbol{S}}^{\alpha},\ad_{\bar{\mathbf{K}}_{\alpha}\boldsymbol{\eta}}\boldsymbol{\kappa}\Big>=-\Big<\bar{\boldsymbol{S}}^{\alpha},\ad_{\boldsymbol{\kappa}}\bar{\mathbf{K}}_{\alpha}\boldsymbol{\eta}\Big>=-\Big<\ad_{\boldsymbol{\kappa}}^{*}\bar{\boldsymbol{S}}^{\alpha},\bar{\mathbf{K}}_{\alpha}\boldsymbol{\eta}\Big>=-\Big<\tilde{\ad}_{\bar{\boldsymbol{S}}^{\alpha}}\boldsymbol{\kappa},\bar{\mathbf{K}}_{\alpha}\boldsymbol{\eta}\Big>.
\end{align}
Substituting the last term of Eq. (\ref{EP22}) into (\ref{EP2}) completes the proof.
\section{Proof of Remark 2}\label{Remark2}
We define the skew-symmetric part of $\left.\dfrac{\partial G_{int}}{\partial \varepsilon}
\right|_{\varepsilon=0}$ calculated in Proposition \ref{prop:1} by changing the role of twists $\boldsymbol{\kappa}$ and $\boldsymbol{\eta}$, as follows:
\begin{align}\label{E69}
\begin{split}
 \textbf{Skew}\Big[\left.\dfrac{\partial G_{int}}{\partial  \varepsilon}\right|_{\varepsilon=0}&\Big]
\coloneqq\dfrac{1}{2}\Big(\left.\dfrac{\partial}{\partial \varepsilon}\right|_{\varepsilon=0}
\!\!\!\!G_{int}(\mathbf{g}_{\varepsilon},\boldsymbol{\kappa})-\left.\dfrac{\partial}{\partial \varepsilon}\right|_{\varepsilon=0}
\!\!\!\!G_{int}(\mathbf{g}_{\varepsilon},\boldsymbol{\eta})\Big)\\
&=\dfrac{1}{2}\Big(\int_{\mathcal{A}_{0}}\Big(\Big<\tilde{\ad}_{\bar{\boldsymbol{S}}^{\alpha}}\boldsymbol{\kappa},\bar{\mathbf{K}}_{\alpha}\boldsymbol{\eta}\Big>-\Big<\tilde{\ad}_{\bar{\boldsymbol{S}}^{\alpha}}\boldsymbol{\eta},\bar{\mathbf{K}}_{\alpha}\boldsymbol{\kappa}\Big>\Big)\bar{j}_{0}d\mathcal{A}\Big)
\end{split}
\end{align}
By Substituting Eq. (\ref{E66}) into (\ref{E69}) we can write
\begin{align}\label{E69R}
\begin{split}
 \textbf{Skew}\Big[\left.\dfrac{\partial G_{int}}{\partial  \varepsilon}\right|_{\varepsilon=0}\Big]
&=\dfrac{1}{2}\Big(\int_{\mathcal{A}_{0}}\Big(\Big<\tilde{\ad}_{\bar{\boldsymbol{S}}^{\alpha}}\boldsymbol{\kappa},\bar{\mathbf{K}}_{\alpha}\boldsymbol{\eta}\Big>+\Big<\tilde{\ad}_{\bar{\boldsymbol{S}}^{\alpha}}\boldsymbol{\eta},\bar{\mathbf{K}}_{\alpha}\boldsymbol{\kappa}\Big>\Big)\bar{j}_{0}d\mathcal{A}\Big)\\
&=\dfrac{1}{2}\Big(\int_{\mathcal{A}_{0}}\Big(\Big<\bar{\boldsymbol{S}}^{\alpha},\dfrac{\partial}{\partial \xi^{\alpha}}\ad_{\boldsymbol{\kappa}}\boldsymbol{\eta}\Big>+\Big<\ad^{*}_{\boldsymbol{\zeta}_{\alpha}}\bar{\boldsymbol{S}^{\alpha}},\ad_{\boldsymbol{\kappa}}\boldsymbol{\eta}\Big>\Big)\bar{j}_{0}d\mathcal{A}\Big)\\
\end{split}
\end{align}
Integration by parts of Eq. (\ref{E69}) yields
\begin{align}\label{E70}
 \textbf{Skew}\Big[\left.\dfrac{\partial G_{int}}{\partial  \varepsilon}\right|_{\varepsilon=0}\Big]=\dfrac{1}{2}\int_{\mathcal{A}_{0}}\Big<\Big(-\dfrac{1}{\bar{j_{0}}}\dfrac{\partial (\bar{j}_{0}\bar{\boldsymbol{S}}^{\alpha})}{\partial \xi^{\alpha}}+\ad^{*}_{\boldsymbol{\zeta}_{\alpha}}\bar{\boldsymbol{S}^{\alpha}}\Big),\ad_{\boldsymbol{\kappa}}\boldsymbol{\eta}\Big>\bar{j}_{0}d\mathcal{A}
\end{align}
It is evident that the first term on the right-hand side of Eq. \eqref{E70} corresponds to the internal virtual work contribution under static conditions in the absence of external forces. At equilibrium, this term vanishes, implying that the associated skew-symmetric component of the stress resultants must be zero.

\bibliography{sn-bibliography}

\end{document}